\documentclass[9pt,shortpaper,twoside,web]{ieeecolor}

\usepackage{generic}
\usepackage{graphicx}
\usepackage[utf8]{inputenc}

\usepackage{comment}

\usepackage[utf8]{inputenc} 
\usepackage[T1]{fontenc}    
\usepackage{hyperref}       
\usepackage{url}            
\usepackage{booktabs}       
\usepackage{amsfonts}       
\usepackage{nicefrac}       
\usepackage{microtype}      
\usepackage{graphicx}
\usepackage[utf8]{inputenc}
\usepackage{amsmath,amssymb}
\usepackage{float}
\usepackage{cite}
\usepackage{epsfig}
\usepackage{graphics}
\usepackage{graphicx}
\usepackage{epstopdf}
\usepackage{subcaption}
\usepackage{pgfplots}
\usepackage{algorithm} 
\usepackage{algpseudocode} 
\usepackage{pifont}
\usepackage[utf8]{inputenc}
\usepackage{amsmath,amssymb}
\usepackage{epsfig}
\usepackage{graphics}
\usepackage{graphicx}
\usepackage{color}
\usepackage{latexsym}
\usepackage{epstopdf}
\usepackage{subcaption}
\usepackage{algorithm}
\usepackage{algpseudocode}
\usepackage{pifont}
\usepackage[utf8]{inputenc}
\usepackage[english]{babel}
\usepackage{bbm}
\usepackage{amsthm}
\newtheorem{theorem}{Theorem}[]
\newtheorem{corollary}{Corollary}[]

\newtheorem{proposition}{Proposition}[]

\newtheorem{remark}{Remark}

\title{
Adaptive KL-UCB based Bandit Algorithms for Markovian and i.i.d. Settings}

\author{Arghyadip~Roy,~\IEEEmembership{Member,~IEEE,}
        Sanjay~Shakkottai,~\IEEEmembership{Fellow,~IEEE}
        ~and
        R. Srikant,~\IEEEmembership{Fellow,~IEEE}
        
  \thanks{A. Roy is with Mehta Family School of Data Science \& Artificial Intelligence, Indian Institute of Technology Guwahati, India 781039 (e-mail: arghyadip@iitg.ac.in).
  S. Shakkottai  is  with  the  Department  of Electrical  and  Computer  Engineering,  The  University  of  Texas  at  Austin, Austin, TX 78712 USA    (e-mail:shakkott@austin.utexas.edu).
R. Srikant is with the Department of Electrical and Computer Engineering, and Coordinated Science Laboratory, University of Illinois at Urbana-Champaign, Champaign, IL 61801 USA (e-mail: rsrikant@illinois.edu).}}

\usepackage{cite}

\allowdisplaybreaks
\begin{document}
\maketitle

\begin{abstract}

In the regret-based formulation of Multi-armed Bandit (MAB) problems, except in rare instances, much of the literature focuses on arms with i.i.d. rewards. In this paper, we consider the problem of obtaining regret guarantees for MAB problems in which the rewards of each arm form a Markov chain which may not belong to a single parameter exponential family. To achieve a logarithmic regret in such problems is not difficult: a variation of standard \textcolor{black}{Kullback-Leibler Upper Confidence Bound (KL-UCB)} does the job. However, the constants obtained from such an analysis are poor for the following reason: i.i.d. rewards are a special case of Markov rewards and it is difficult to design an algorithm that works well independent of whether the underlying model is truly Markovian or i.i.d. To overcome this issue, we introduce a novel algorithm that identifies whether the rewards from each arm are truly Markovian or i.i.d. using a
\textcolor{black}{total variation}
distance-based test. Our algorithm then switches from using a standard KL-UCB to a specialized version of KL-UCB when it determines that the arm reward is Markovian, thus resulting in low regrets for both i.i.d. and Markovian settings.

\end{abstract}

\section{Introduction}
In a Multi-armed Bandit (MAB) problem \cite{robbins1952some,lai1985asymptotically}, 
a player needs to choose arms sequentially
to maximize the total reward or minimize the total regret \cite{auer2002nonstochastic}. 
Unlike i.i.d. arm reward case \cite{thompson1933likelihood,lai1985asymptotically,anantharam1987asymptotically-1,agrawal1995sample,auer2002finite,garivier2011kl,cappe2013kullback},
there are only a few works which consider bandits with Markovian rewards. When the states of the arms which are not played remain frozen (\textit{rested} bandit \cite{tekin2010online,anantharam1987asymptotically-2,tekin2012online,moulos2020finite}), the states observed in their next selections do not depend on the interval between successive plays of the arms. 
In the \textit{restless} \cite{liu2012learning,tekin2011online,tekin2012online} case, states of the arms continue to evolve irrespective of the selections by the player. 
In this paper, we focus on the rested bandit setting. Many real-world problems such as gambling, ad placement and clinical trials fall into this category \cite{cortes2017}. 
If a slot-machine (viewed as an arm) produces high reward in a particular play, then the probability that it will produce \textcolor{black}{high} reward in the next play is \textcolor{black}{not} very high. Therefore, it is reasonable to assume that the reward distribution depends on the previous outcome and can be modeled through the rested bandit setting. Another example is in an online advertising setting to a customer (e.g. video \textcolor{black}{advertisements (ads)} with IMDb TV, social ads on Facebook), where an agent presents an ad to a customer from a pool of ads and adapts future displayed ads based on the customer's past responses. 
\textcolor{black}{Other examples are medical diagnosis problem (repeated interventions of different medications), click-through rate prediction and rating
systems \cite{graepel2010web,herbrich2006trueskill}.}\par  
Existing studies in the literature on rested bandits either consider a particular class of Markov chains (with single-parameter families of transition matrices) \cite{anantharam1987asymptotically-2,moulos2020finite}, or result in large asymptotic regret bounds \cite{tekin2010online,tekin2012online}. In general, however, these Markov chains may not belong to a single parameter family of transition matrices; nevertheless, we would desire low regret. 
In the setting considered by us, we assume that each arm evolves according to a two-state Markov chain. When an arm is chosen, a finite reward is generated based on the present state of the arm.
Imagine a scenario where a company (Netflix, say \cite{lattimore2018bandit})
recommends a set of products (places movies in user's browser) to users (subscribers) to maximize the generated
revenue. A user’s response to a recommended product (a movie genre) can be captured using a two state (corresponding to negative and positive ratings) Markov
chain where buying a product (providing a positive rating) generates a positive reward, else the reward is zero. A user's rating for a movie may be temporally correlated based on the genre. Using a MAB approach, the company can adapt the set of displayed products \cite{avadhanula2021stochastic}.
In such a setting, we can obtain a logarithmic regret by a variant of standard  KL-UCB \cite{garivier2011kl,cappe2013kullback} (originally designed for i.i.d. rewards). Although the constant obtained from the regret analysis is 
better than that of \cite{tekin2010online}, the performance is poor for i.i.d. rewards (special case of Markovian rewards). We address this problem in this paper.
The main contributions of this paper are as follows.\par
1. \textbf{\textcolor{black}{Adaptive} KL-UCB Algorithm}: We propose a novel \textcolor{black}{Total Variation} KL-UCB (\textcolor{black}{TV}-KL-UCB) algorithm which can identify whether the rewards from an arm are truly Markovian or i.i.d. The identification is done using a (time-dependent) \textcolor{black}{total variation} distance-based test between the \textcolor{black}{empirical} estimates of transition probabilities from one state to another.   
The proposed algorithm switches from standard sample mean based KL-UCB to a sample transition probability based KL-UCB when it determines that the arm
reward is Markovian using the test. 
The sample transition probability based KL-UCB determines the upper confidence bound on the transition probability from the current state of the arm and uses the estimate of the transition probability to the current state while evaluating the mean reward. 
An arm can be represented uniquely by the transition probability matrix (two parameters), however, requiring a single parameter (mean reward) in an i.i.d setting.
Our approach adapts itself from learning two parameters for truly Markovian arms to learning a single parameter for i.i.d. arms, thereby producing low regrets in both cases. An extension to multi-state Markovian arms is also described.\par 
2. \textbf{Upper Bound on Regret}:  We derive finite time and asymptotic upper bounds on the regret of TV-KL-UCB (capturing the worst-case regret \cite{lattimore2018bandit}) for both truly Markovian and i.i.d. rewards. We prove that \textcolor{black}{TV}-KL-UCB is order-optimal for Markovian rewards and optimal when all arm rewards are i.i.d. \par
The standard analysis for regret with KL-UCB (either for i.i.d. rewards \cite{garivier2011kl}, or rewards from a Markov transition matrix  specified through a single-parameter exponential family \cite{anantharam1987asymptotically-1,anantharam1987asymptotically-2}) crucially relies on the invertibility of the KL divergence function when applied to \textcolor{black}{empirical} estimates. This allows one to translate concentration guarantees for sums of random variables to one for level crossing of the KL divergence function applied to empirical estimates. In our multi-parameter estimation setting, this invertibility property no longer directly holds. 
However, we convert this multi-parameter problem 
into a collection of single parameter problems 
and 
derive concentration bounds for individual single parameter problems
which are combined to obtain an upper bound on the regret of \textcolor{black}{TV}-KL-UCB.  
To derive our bound, we establish that a certain condition on the \textcolor{black}{total variation} distance between estimates of transition probabilities (for using a sample mean based KL-UCB) is satisfied infinitely often iff the arm rewards are i.i.d. over time, which in turn implies that the regret due to choosing the incorrect variant of KL-UCB vanishes asymptotically.

Analytical and experimental results establish that \textcolor{black}{TV}-KL-UCB performs better than the state-of-the-art algorithms \cite{tekin2010online,moulos2020finite} when at least one of the arm rewards is truly Markovian. Moreover, \textcolor{black}{TV}-KL-UCB is optimal \cite{garivier2011kl} when all arm rewards are i.i.d.\par
\textbf{Related Work:}
While much of the literature on MAB focuses on arms with i.i.d. rewards, arms with Markovian rewards have not been studied extensively. 
In \cite{lai1985asymptotically}, when the parameter space is dense and can be represented using a single-parameter density function, 
a lower bound on the regret is derived.
The authors in \cite{lai1985asymptotically} also propose policies that asymptotically achieve the lower bound. The work in \cite{lai1985asymptotically} is extended in \cite{anantharam1987asymptotically-1} for the case when multiple arms can be played at a time. A sample mean based index policy in \cite{agrawal1995sample} achieves a logarithmic regret for one-parameter family of distributions. 
The KL-UCB based index policy \cite{garivier2011kl,cappe2013kullback} is asymptotically optimal for Bernoulli rewards and 
performs better than the UCB policy \cite{auer2002finite}. \par 
\cite{zheng2016sequential} provides an overview of the state-of-the-art on Markovian bandits. 
Under the assumption of single-parameter families of transition matrices, an index policy which matches the corresponding lower bound asymptotically, is proposed in \cite{anantharam1987asymptotically-2}.
In \cite{tekin2010online}, a UCB policy based on the sample mean reward is proposed. Unlike \cite{anantharam1987asymptotically-2}, the analysis in \cite{tekin2010online} is not restricted to single-parameter family of transition matrices. Moreover, since the index calculation is based on the sample mean, the policy is significantly simpler than that of \cite{anantharam1987asymptotically-2}. Although order-optimal, the proposed policy may be worse than that in \cite{anantharam1987asymptotically-2} in terms of the constant. In \cite{moulos2020finite}, a straightforward extension of KL-UCB using sample mean is proposed for the optimal allocation problem involving multiple plays. 
Similar to \cite{anantharam1987asymptotically-2}, 
rewards are generated from Markov chains belonging to a one-parameter exponential family.\par  
\textcolor{black}{Unlike i.i.d. rewards, in many practical applications, temporal variation in reward distribution is present \cite{gittins1974dynamic,whittle1988restless,lazaric2014online}
ranging from Markovian to general history-dependent rewards and
adversarial rewards 
\cite{foster1999regret,cesa2006prediction}. In \cite{besbes2014stochastic}, regret analysis for a large class of MAB problems involving non-stationary rewards is conducted. A relation between the rate of variation in reward (limited by a variation budget unlike the unbounded adverserial setup, however including  a large set of non-stationary stochastic MABs) and minimal regret is established in \cite{besbes2014stochastic}.  Low-regret algorithms are developed to learn the optimal policy in Markov Decision Processes (MDPs).  UCRL2 \cite{jaksch2010near} switches between  computing the optimal policy with the largest optimal gain in each phase for the MDP and implementing the policy, based on a state-action visit criteria. \cite{azar2017minimax,fruit2018near} improve upon UCRL2 by demonstrating better finite time behavior and weaker dependence on diameter and state space of the MDP, respectively.}\par
\textbf{\textcolor{black}{Our Contributions:}}
\textcolor{black}{Existing works on rested
Markovian bandits either consider only single-parameter families of transition matrices
(one-parameter exponential family of Markov chains) \cite{anantharam1987asymptotically-2,moulos2020finite} or undertake a sample
mean based approach \cite{tekin2010online,tekin2012online} which is closer in spirit to i.i.d bandits than Markovian
bandits. We, for the first time, consider a sample transition probability based approach
combined with a sample mean based approach in \cite{tekin2010online,tekin2012online} with a significant improvement in
performance than that of \cite{tekin2010online,tekin2012online}.}
Following \cite{tekin2010online} and unlike \cite{anantharam1987asymptotically-2, moulos2020finite}, we do not consider any parameterization on the transition probability matrices. The only assumption we require is that the Markov chains have to be irreducible. The key reasons why our approach achieves lower regret than \cite{tekin2010online} are (i) unlike us (confidence bounds on sample transition probabilities for truly Markovian arms coupled with a sample mean based approach for i.i.d arms), \cite{tekin2010online} always uses sample mean-based indices which may not uniquely represent the arms in a truly Markovian setting, and (ii) usage of KL-UCB (based on the Chernoff’s bound) provides a tighter confidence bound than the Hoeffding’s bound for UCB \cite{lattimore2018bandit}.
\textcolor{black}{
Our analysis is different from that of
the Bernoulli bandit \cite{garivier2011kl} in the following way: (i) we determine the contribution of mixing time components of the underlying Markov chains in the
finite-time regret upper bound,
(ii) we prove that the appropriate conditions for the online test are
satisfied infinitely often for both cases (i.i.d and purely Markovian). }
Since we do not consider only an one-parameter exponential family of Markov chains, it is unclear that whether our algorithm is the best possible algorithm for this setting. To show that an algorithm is asymptotically optimal, one needs to obtain matching upper and lower bounds which remains an open problem. However, we note that both our theory and simulations show that our results improve upon the performance of the state-of-the-art algorithms.
\section{ Problem Formulation \& Preliminaries}\label{sec:sysmod}
We assume that we have $K$ arms. The reward from each arm is modeled as a two-state irreducible Markov chain with a state space $\mathcal{S}=\{0,1\}$. Let the reward obtained when \textcolor{black}{an} arm 
which is in state $s$, is played be denoted by \textcolor{black}{$r(s)=s$ (say).}
Let the transition probability from state $s=0 (s=1)$ to state $s=1 (s=0)$ 
of arm $i$ be denoted by $p_{01}^i (p_{10}^i)$. 
Let the stationary distribution of arm $i$ be ${\pi}_i=(\pi_i(s),s \in \mathcal{S})$. Therefore, the mean reward of arm $i$ ($\mu_i$, say) is
   $ \mu_i=\sum_{s\in \mathcal{S}}s\pi_i(s)=\pi_i(1)$.
Let $\mu^*(=\max_{1\le i \le K} \mu_i)$ denote the mean reward of the best arm. W.l.o.g, we assume that $\mu^*=\mu_1$. 
Let the suboptimality gap of arm $i$ be $\Delta_i=\mu_1-\mu_i$. Let $R_{\alpha}(n)$ denote the regret of policy $\alpha$  up to time $n$. 
\textcolor{black}{Hence,
 $   R_{\alpha}(n)=n\mu_1-\mathbb{E}_{\alpha}[\sum_{t=1}^n r(s(\alpha(t)))]$,}
where $\alpha(t)$ denotes the arm (which is in state $s(\alpha(t))$, say) selected at time $t$ by policy $\alpha$. 
A policy $\alpha$ is said to be \textit{uniformly good} if $R_{\alpha}(n)=o(n^{\beta})$ for every $\beta>0$. 
For two arms $i$ and $j$ with associated Markov chains $M_i$ and $M_j$ (say), the KL distance between them is 
    $I(M_i||M_j)= \pi_i(0) D(p_{01}^i||p_{01}^j)+\pi_i(1) D(p_{10}^i||p_{10}^j)$,
where $D(A||B)=A \log \frac{A}{B}+(1-A) \log \frac{1-A}{1-B}$.
In \cite{anantharam1987asymptotically-2}, a lower bound on the regret of any uniformly good policy is derived. For a uniformly good policy $\alpha$, 
$\liminf \limits_{n \to \infty} \frac{R_{\alpha}(n)}{\log n}\ge \sum_{i=2}^K \frac{\Delta_i}{I(M_i||M_1)}$.
In \cite{anantharam1987asymptotically-2}, the above lower bound is derived when the transition functions belong to a single-parameter family. It is straightforward to show that the lower bound holds more generally but
since the proof techniques are standard, we present the proof in Appendix \ref{app:a}. 
We aim to determine an upper bound on $R_{\alpha}(n)$ as a function of $n$ for a given policy $\alpha$.
\begin{remark}
\textcolor{black}{The model considered by us is a first step towards capturing the temporal correlation among decisions in a MAB
problem. More complicated models involving more than two states would require us
to estimate more parameters, leading to a
higher complexity
which can be determined using statistical learning theory.
Because of the trade-off between model complexity and
usefulness, we take into account a simple Markovian model with two states.} The extension to multi-state Markovian model is described in Section \ref{sec:multistate}.
\end{remark}
\begin{remark}
Ideally, to capture the correlation among decisions, one requires a
Hidden Markov Model (HMM) where the state of the underlying
Markov chain is not observable, and only the reward obtained by
choosing an arm is observable. 
The cardinality of the state space of the underlying HMM is typically unknown and hence, acts as
an additional hyperparameter to be estimated. 

\end{remark}
\section{TV-KL-UCB Algorithm \& Regret Upper Bound}\label{sec:ub}
When the rewards of an arm are i.i.d., the arm can be represented uniquely using the mean reward. However, in the truly Markovian reward setting, arm $i$ can be described uniquely by $p_{01}^i$ and $p_{10}^i$. Using a variation (we call it KL-UCB-MC, see Appendix \ref{app:klucbmc}) of standard KL-UCB for i.i.d. rewards \cite{garivier2011kl,cappe2013kullback}, one can obtain a logarithmic regret. The main idea is to obtain a confidence bound on the estimate of $p_{01}^i$ (estimate of $p_{10}^i$) and use the estimate of $p_{10}^i$ (estimate of $p_{01}^i$) in state $0$ (state $1$) of arm $i$ using KL-UCB. For purely Markovian arms, the resulting regret is smaller than the regret of the algorithm in \cite{tekin2010online}. However, KL-UCB-MC results in large constants in the regret for i.i.d. rewards. 
Hence, we introduce the \textcolor{black}{TV}-KL-UCB algorithm which improves over KL-UCB-MC and performs well in both truly Markovian and i.i.d. settings. 
\subsection{Total Variation KL-UCB Algorithm}
\begin{algorithm}
\caption{\textcolor{black}{Total Variation} KL-UCB Algorithm (\textcolor{black}{TV}-KL-UCB)}\label{algo:1}
\label{NCalgorithm}
\begin{algorithmic}[1]
\small
\State Input $K$ (number of arms).
\State Choose each arm once.
\While{TRUE} 

\If {
$\textcolor{black}{(|{\hat{p}}_{01}^i(t-1)+{\hat{p}}_{10}^i(t-1)-1|)}>\frac{1}{{(t-1)}^{1/4}}$  (\textbf{procedure} STP\_PHASE)}
 \If { (state of arm $i=0$)}
\begin{equation}\label{eq:state0}
\begin{aligned}
U_i= \sup\{\frac{\tilde{p}}{\tilde{p}+{\hat{p}}_{10}^i(t-1)}:D({\hat{p}}_{01}^i(t-1)||\tilde{p})\le\frac{\log f(t)}{T_i(t-1)}\}.
  \end{aligned}  
\end{equation}
\Else
\begin{equation}\label{eq:state1}
\begin{aligned}
 U_i= \sup\{\frac{{\hat{p}}_{01}^i(t-1)}{{\hat{p}}_{01}^i(t-1)+\tilde{q}}:D({\hat{p}}_{10}^i(t-1)||\tilde{q})\le\frac{\log f(t)}{T_i(t-1)}\}.
  \end{aligned}  
\end{equation}
\EndIf
\Else { (\textbf{procedure} SM\_PHASE)}
\begin{equation}\label{eq:stateless}
\begin{aligned}
U_i=\sup\{\tilde{\mu}\in [{\hat{\mu}}^i(t-1),1]:D({\hat{\mu}}^i(t-1)||\tilde{\mu})\le\frac{\log f(t)}{T_i(t-1)}\}.
\end{aligned}
\end{equation}
\EndIf
\State Choose $A_t=\arg \max \limits_i U_i$.
\EndWhile
\end{algorithmic}
\end{algorithm}
The proposed \textcolor{black}{TV}-KL-UCB which is based on sample transition probabilities between different states and sample mean, is motivated from the KL-UCB algorithm \cite{garivier2011kl}. 
Let $T_{i,j}(t)$ denote the number of times arm $i$ is selected while it was in state $j$, till time $t$. We \textcolor{black}{define}
$T_i(t):=\sum_{j\in \mathcal{S}} T_{i,j}(t)$.
We further assume that 
${\hat{p}}_{01}^i(t)$, ${\hat{p}}_{10}^i(t)$ and ${\hat{\mu}}^i(t)$
denote the \textcolor{black}{empirical} estimate of $p_{01}^i$, \textcolor{black}{empirical} estimate of $p_{10}^i$ and sample mean of arm $i$ at time $t$, respectively (\textcolor{black}{see Appendix \ref{app:theo}}). 
We compute the 
\textcolor{black}{total variation}
distance (\textcolor{black}{$TV(\cdot||\cdot)$, say}) between
${\hat{p}}_{01}^i(t-1)$ and $1-{\hat{p}}_{10}^i(t-1)$ of arm $i$ (\textcolor{black}{which is $|{\hat{p}}_{01}^i(t-1)+{\hat{p}}_{10}^i(t-1)-1|$}). If it is greater than $\frac{1}{{(t-1)}^{1/4}}$ (Line 4), then
based on the current state of the arm, we calculate the index of the arm (STP\_PHASE). If arm $i$ is in state $0$, then the index is calculated using the confidence bound on the estimate of $p_{01}^i$ at time $t$ (Line 5), else using the confidence bound on the estimate of $p_{10}^i$ (Line 6). \textcolor{black}{$U_i$ is an overestimate of $\mu_i$ with high probability because of the considered confidence bounds (Lines 5-6). Therefore, a suboptimal arm cannot be played too often as $U_1$ overestimates $\mu_1$ and a suboptimal arm $i$ can be played only if $U_i>U_1$. The index computation ensures that an arm is explored more often if it is promising (high $\frac{{\hat{p}}_{01}^i(t-1)}{{\hat{p}}_{01}^i(t-1)+{\hat{p}}_{10}^i(t-1)}$) or under-explored (small $T_i(t)$).}
We use the estimate of the transition probability from the other state while evaluating the index. However, if
\textcolor{black}{$|{\hat{p}}_{01}^i(t-1)+{\hat{p}}_{10}^i(t-1)-1|$} is less than $\frac{1}{{(t-1)}^{1/4}}$ , then the index of the arm is calculated (Line 8) using the confidence bound on the current value of ${\hat{\mu}}^i$ (SM\_PHASE).
Then, we play the arm with the highest index. 
We take $f(t)=1+t \log^2(t)$. 
The physical interpretation behind the condition on \textcolor{black}{total variation} distance is that the KL-UCB algorithm (which uses sample mean) \cite{garivier2011kl} is asymptotically optimal for i.i.d. Bernoulli arms. When ${{p}}_{10}^i=1-{{p}}_{01}^i$ (i.i.d arm), the condition in Line 8 is satisfied frequently often, and arm $i$ uses  (\ref{eq:stateless}) for index calculation, similar to \cite{garivier2011kl}.
Else, the condition in Line 4 is met frequently often. We formally establish these statements in Proposition \ref{lemma:Hellinger}. 
\begin{remark}
\textcolor{black}{
The key differences between \textcolor{black}{TV}-KL-UCB and KL-UCB \cite{garivier2011kl} are:
(1) the idea of using a confidence bound on one of the transition probabilities while using the raw value of the other one to determine the index of an arm is novel, 
(2) the design of an online test for detecting whether an arm is purely Markovian or not is
new, allowing us to simultaneously use the sample transition probability and the sample mean to uniquely
represent truly Markovian arms and i.i.d arms, respectively.}
\end{remark}
\begin{remark}
\textcolor{black}{In general, it may not be possible to write
closed-form expressions for (\ref{eq:state0})-(\ref{eq:stateless}) which determine zeroes of convex and
increasing scalar functions \cite{cappe2013kullback}. This can be performed by dichotomic search or Newton
iteration. However, for finitely supported distribution (as considered by us), it can
be reduced to maximization of a linear function on the probability simplex under KL distance
constraints to obtain an explicit computational solution \cite[Appendix~C1]{cappe2013supplement}.}
\end{remark}
\begin{remark}
Instead of KL distance which is a natural choice for representing the similarity between ${\hat{p}}_{10}^i(t)$ and $1-{\hat{p}}_{01}^i(t)$, 
we choose the 
\textcolor{black}{total variation} distance. 
\textcolor{black}{We can also use other distances such as Hellinger distance}
because it permits additive separability of the estimates: ${\hat{p}}_{10}^i(t)$ and ${\hat{p}}_{01}^i(t),$ which in turn enables the use of standard concentration inequalities for the proof of asymptotic upper bound on regret (See Appendix \ref{app:klucbmc}).
\textcolor{black}{Note that any distance metric $L(\cdot||\cdot)$ which satisfies $L(X ||Y)\le TV (X ||Y)$ for two probability distributions
$X=(x_1,\cdots,x_k)$ and $Y=(y_1,\cdots,y_k)$ (say, and hence, Proposition \ref{lemma:Hellinger} holds) can be used for the online test. Since both Hellinger distance ($H(X||Y)$, say where $2H^2(X||Y)=\sum_{i=1}^k(\sqrt{x_i}-\sqrt{y_i})^2$) and Jensen-Shannon distance \cite{lin1991divergence} satisfy
the above inequality, similar online test can be designed using them. Morover, results
in \cite[Theorem~I.2]{agrawal2020finite} can be utilized to design similar online test using
KL distance.} 
\end{remark}

\subsection{Regret Upper Bound} 
Assuming the regret of \textcolor{black}{TV}-KL-UCB till time $n$ is $R_n$,
the theorem 
provides an upper bound on the regret of \textcolor{black}{TV}-KL-UCB.

\begin{theorem}\label{theo:regret_2}
\textcolor{black}{Regret of Algorithm \ref{algo:1} is bounded by}\\
(a) truly Markovian arms:\\
\textcolor{black} {$R_n
\le \sum \limits_{i\neq 1} \Delta_i
\big(\tau_{1,i} +\frac{20}{{\epsilon_1}^2}+\frac{2}{{\epsilon}_{p}^2}+\frac{2}{{\epsilon}_{q}^2}+
10\sum \limits_{t=1}^{\infty} (t+1)^3  \exp(-2(t-1){\epsilon_1}^2(p_{01}^i+p_{10}^i)^2)+
10\sum \limits_{t=1}^{\infty} (t+1)^3  \exp(-2(t-1){\epsilon_1}^2(p_{01}^1+p_{10}^1)^2)\big).$}\\
\textcolor{black}{$\limsup \limits_{n \to \infty}\frac{R_n}{\log n}\le \sum \limits_{i \neq 1}\Delta_i \big[ 
{\frac{2.\mathbbm{1}\{p_{01}^1 p_{10}^i<{p_{10}^1}\}}{D(p_{01}^i||\frac{p_{01}^1p_{10}^i}{p_{10}^1})}}+ 
\frac{2}{D(p_{10}^i||\frac{p_{10}^1p_{01}^i}{p_{01}^1})}\big].$}\\
(b) i.i.d. optimal and truly Markovian suboptimal arms:\\
\textcolor{black}{$R_n
\le \sum \limits_{i\neq 1}\Delta_i\big(
\tau_{2,i} +\frac{6}{{\epsilon_1}^2}+\frac{2}{{\epsilon}_{\mu}^2}+
6\sum\limits_{t=1}^{\infty} (t+1)^3  \exp(-2(t-1){\epsilon_1}^2(p_{01}^i+p_{10}^i)^2)+
4\sum\limits_{t=1}^{\infty}\exp(-\frac{2}{9}\sqrt{t})+\sum\limits_{t=1}^{\infty}(t+1)^3\exp(-2\sqrt{t-1})\big).$}\\
\textcolor{black}{$\limsup \limits_{n \to \infty}\frac{R_n}{\log n}\le \sum \limits_{i \neq 1}\Delta_i \big[ 
{\frac{\mathbbm{1}\{\mu_1p_{10}^i<{1-\mu_1}\}}{D(p_{01}^i||\frac{\mu_1p_{10}^i}{1-\mu_1})}}+ 
\frac{1}{D(p_{10}^i||\frac{p_{01}^i(1-\mu_1)}{\mu_1})}\big].$}\\
(c) truly Markovian optimal and i.i.d. suboptimal arms:
\textcolor{black}{$R_n
\le \sum \limits_{i\neq 1}
\Delta_i\big(\tau_{3,i} +\frac{7}{{\epsilon_1}^2}+\frac{2}{{\epsilon}_{p}^2}+\frac{2}{{\epsilon}_{q}^2}+
6\sum\limits_{t=1}^{\infty} (t+1)^3  \exp(-2(t-1){\epsilon_1}^2(p_{01}^1+p_{10}^1)^2)+
4\sum\limits_{t=1}^{\infty}\exp(-\frac{2}{9}\sqrt{t})+\sum\limits_{t=1}^{\infty}(t+1)^3\exp(-2\sqrt{t-1})\big).$}

\textcolor{black}{$\limsup \limits_{n \to \infty}\frac{R_n}{\log n}\le \sum \limits_{i \neq 1}
\frac{2\Delta_i}{D(\mu_i||\frac{p_{01}^1}{p_{01}^1+p_{10}^1})}.$}\\
(d) i.i.d. arms:
\textcolor{black}{$R_n
\le \sum_{i\neq 1} \Delta_i
\big(\tau_{4,i} +\frac{1}{2{\epsilon_1}^2}+\frac{2}{{\epsilon}_{\mu}^2}+
8\sum\limits_{t=1}^{\infty}\exp(-\frac{2}{9}\sqrt{t})+2\sum \limits_{t=1}^{\infty}(t+1)^3\exp(-2\sqrt{t-1})\big).$}

\textcolor{black}{$\limsup \limits_{n \to \infty}\frac{R_n}{\log n}\le \sum \limits_{i \neq 1}
\frac{\Delta_i}{D(\mu_i||\mu_1)},$
where values of $\{\tau_{1,i},
\tau_{2,i},\tau_{3,i},\tau_{4,i}\}$ are given in Appendix \ref{app:theo}.}
\end{theorem}
\begin{proof}
The detailed proof is provided in Appendix \ref{app:theo}.
The proof idea is motivated by
the regret analysis for i.i.d. rewards in \cite{garivier2011kl}. However, in this paper,
unlike \cite{garivier2011kl} which deals with a single parameter (mean reward), we convert a multi-parameter estimation (transition probabilities between different states of a Markov chain) problem into a collection of single-parameter estimation problems and derive corresponding concentration bounds. 
Since the proof is technical, 
we first briefly outline the key steps for the truly Markovian arm setting (case (a)). Proof for other cases follow in a similar manner. First, we determine the finite time upper bound on the regret of \textcolor{black}{TV}-KL-UCB.
The key steps in the proof are:\\
1) We first establish that estimates of transition probabilities of the arms are close to the true transition probabilities. Furthermore, we show that the \textcolor{black}{total variation} distance condition for arm $i$ ($\textcolor{black}{TV}({\hat{p}}_{01}^i(t)||1-{\hat{p}}_{10}^i(t))>\frac{1}{{t}^{1/4}}$) is satisfied infinitely often after a sufficiently large time $\tau_{1,i}$. Proof follows from Propositions \ref{lemma:concentratation} and \ref{lemma:Hellinger} (See Appendix \ref{app:theo}).\\
2) We then establish that the confidence bounds
for the estimates of transition probabilities of the optimal arm are never too far from respective true values. We consider two cases, corresponding to states 0 and 1 (Equations (\ref{eq:state0}) and (\ref{eq:state1})) of the optimal arm, respectively. Specifically, we prove that the expected number of times ${\tilde{p}}_{01}^{*1}(t)$ is less than $p_{01}^1-\epsilon_p$ and ${\tilde{p}}_{10}^{*1}(t)$ is more than $p_{10}^1+\epsilon_q$, are upper bounded by $\frac{2}{{\epsilon_p}^2}$ and $\frac{2}{{\epsilon_q}^2}$, respectively. We use Pinsker's inequality, Chernoff's bound and some algebraic manipulations to complete the proof.\\
3)  We prove that when the estimates of transition probabilities of the arms are close to the true values, the appropriate condition on the \textcolor{black}{total variation} distance is satisfied and confidence bounds for the estimates of transition probabilities of the optimal arm are close to the corresponding true values, then
the index associated with a sub-optimal arm is not often much greater than the index of the optimal arm. We consider four cases for different state-arm combinations of a given sub-optimal arm and the optimal arm. We illustrate the proof sketch when the both arms are in state $0$. 
We prove that $\sum \limits_{t=\tau_{1,i}}^n\mathbb{P}\{\frac{{\tilde{p}}_{01}^{*1}(t)}{{\tilde{p}}_{01}^{*1}(t)+{\hat{p}}_{10}^1(t)}<\frac{{{\tilde{p}}_{01}}^{*i}(t)}{{{\tilde{p}}_{01}}^{*i}(t)+{{\hat{p}}_{10}^i}(t)}\}$ is finite. When the conditions stated above are true, this is equivalent to proving 
that $\sum \limits_{t=\tau_{1,i}}^n\mathbb{P}\{D(p_{01}^i+\epsilon_1||\frac{(p_{01}^1-\epsilon_p) (p_{10}^i-\epsilon_1)}{(p_{10}^1+\epsilon_1)})\le \frac{\log f(n)}{t}\}$ is finite. 
The rest of the
proof uses the monotonicity property of $D(x||y)$ for $x<y$. We complete the proof by choosing an appropriate $\tau_{1,i}$ as a function of $n$.
Then, we derive an asymptotic upper bound on the regret of \textcolor{black}{TV}-KL-UCB by selecting  $\epsilon_1=\epsilon_p=\epsilon_q=\log^{-1/4} (n)$. 
\qed  
\end{proof}
Multiplicative factors  $\mathbbm{1}\{p_{01}^1 p_{10}^i<{p_{10}^1}\}$ and $\mathbbm{1}\{\mu_1p_{10}^i<{1-\mu_1}\}$ are required because the corresponding $D(.||.)$ do not exist when $\frac{p_{01}^1 p_{10}^i}{p_{10}^1}$ and
$\frac{\mu_1p_{10}^i}{1-\mu_1}$
become more than or equal to $1$.
Note that the upper bound on the regret of \textcolor{black}{TV}-KL-UCB matches the lower bound \cite{garivier2011kl} when all arm rewards are i.i.d. Therefore, our algorithm is optimal for i.i.d. arm rewards.
The next theorem analytically verifies that our bounds are better than the state-of-the-art for a large class of problem parameters. The proof is given in Appendix \ref{app:234}. The cases not covered in the theorem are explored in simulations later.
\begin{theorem}\label{theo:eigen}
Let the eigenvalue gap of $i^{\rm{th}}$ arm be $\sigma_i$. The
asymptotic upper bound on the regret of \textcolor{black}{TV}-KL-UCB is smaller than that of \cite{tekin2010online} (we call it UCB-SM)
always (when $\min \limits_{i}\sigma_i\ge \frac{1}{1440}$) for i.i.d (truly Markovian suboptimal) arms. 
\end{theorem}
\subsection{Extension to Multi-state Markovian Model }\label{sec:multistate}
In this section, we sketch the modifications required to take into account finite-state space ($\mathcal{X}$, say with state space $\{0,1,\cdots,x\}$ and $|\mathcal{X}|>2$) Markov chain to represent the arm rewards. Let the transition probability from state $a \in \mathcal{X}$ to state $b \in \mathcal{X}$ for arm $i$ be $p^i_{ab}$. Recall that the mean reward and stationary distribution associated with arm $i$ are $\mu_i$ and $\pi_i$, respectively. 
Now, for a Markov chain, a close form solution for $\pi_i$ can be obtained by solving the global balance equation corresponding to arm $i$. Let 
$$ \boldsymbol{p}^i_s=[p^i_{s0},p^i_{s1},
\cdots,p^i_{sx}],$$ $$ \boldsymbol{\hat{p}}^i_s(t)=[\hat{p}^i_{s0}(t),\hat{p}^i_{s1}(t),
\cdots,\hat{p}^i_{sx}(t)],$$ and
\begin{equation*}
\begin{split}
&\pi_i(s)=f_s(\boldsymbol{p}^i_0,\boldsymbol{p}^i_1,\cdots, \boldsymbol{p}^i_x). 
 \end{split}
\end{equation*}
We have, $\mu_i=\sum_{s \in \mathcal{X}} r(s) \pi_i(s)$.
The complete description of the resulting algorithm for the finite state space Markov chain with multiple states (we call it m-TV-KL-UCB) is given in Algorithm \ref{algo:multi}.
Similar to Algorithm \ref{algo:1}, the algorithm is divided into two phases, viz., STP\_PHASE and SM\_PHASE, depending on whether the condition on the total variation distance is satisfied or not. The condition on the total variation distance (for the online test) for arm $i$ is $$TV(\boldsymbol{\hat{p}}^i_s(t)||\boldsymbol{\hat{p}}^i_{s'}(t))<\frac{1}{t^{1/4}}, \forall s,s' \in \{0,1,\cdots,x\},s\neq s'.$$ When this condition is met ($flag=0$ in Algorithm \ref{algo:multi}) at time $t$, then we conclude that the arm $i$ is i.i.d and hence, choose the SM\_PHASE.  Else, STP\_PHASE is chosen. 
The motivation behind this condition is as follows. Arm $i$ is i.i.d. (a special 
case of Markovian arm) if
$\boldsymbol{p}^i_s\,{\buildrel d \over =}\,\boldsymbol{p}^i_{s'}, \forall s,s'\in \mathcal{X}, s\neq s'.$
If the condition (Line 12) is satisfied, then the algorithm enters the SM\_PHASE, and the set of operations remains identical to that of Algorithm 1. 
However, if this condition is not satisfied, then the algorithm enters the STP\_PHASE as the arm is likely to be a truly Markovian arm. Now, if the current state of the Markov chain is $s$, then we compute $U_i$ in the following way. 
 $$U_i=\sup \{\tilde{\mu}_{i,s}(t):D(\boldsymbol{\hat{p}}^i_{s}(t-1)||\tilde{\boldsymbol{p}}^i_{s})\le \frac{\log f(t)}{T_i(t-1)}\},$$
 where $\tilde{\mu}_{i,s}(t)=\sum_{s \in \mathcal{X}} r(s) f_s(\boldsymbol{\hat{p}}^i_0(t-1),\cdots, \tilde{\boldsymbol{p}}^i_s,\cdots, \boldsymbol{\hat{p}}^i_x(t-1))$.
The rest of the algorithm is identical to Algorithm \ref{algo:1}. 
\begin{algorithm}
\caption{m-TV-KL-UCB }\label{algo:multi}
\label{NCalgorithm}
\begin{algorithmic}[1]
\small
\State Choose each arm once.
\While{TRUE} 
\State Set $flag\leftarrow 0$.
\For {$s,s'=0$ to $(x-1)$, $s\neq s'$}
\If {$TV(\boldsymbol{\hat{p}}^i_s(t-1)||\boldsymbol{\hat{p}}^i_{s'}(t-1))>\frac{1}{(t-1)^{1/4}}$} 
\State $flag\leftarrow 1$.
\EndIf
\EndFor
\If {
$(flag==1)$  (\textbf{procedure} STP\_PHASE)}
 \If { (state of arm $i=s$)}
\begin{equation}\label{eq:state_s}
\begin{aligned}
U_i= \sup \{\tilde{\mu}_{i,s}(t):D(\boldsymbol{\hat{p}}^i_{s}(t-1)||\tilde{\boldsymbol{p}}^i_{s})\le \frac{\log f(t)}{T_i(t-1)}\}.
  \end{aligned}  
\end{equation}
\EndIf
\Else { (\textbf{procedure} SM\_PHASE)}
\begin{equation}\label{eq:stateless_x}
\begin{aligned}
U_i=\sup\{\tilde{\mu}\in [{\hat{\mu}}^i(t-1),1]:D({\hat{\mu}}^i(t-1)||\tilde{\mu})\le\frac{\log f(t)}{T_i(t-1)}\}.
\end{aligned}
\end{equation}
\EndIf
\State Choose $A_t=\arg \max \limits_i U_i$.
\EndWhile
\end{algorithmic}
\end{algorithm}
\section{Experimental Evaluation }\label{sec:result}
\begin{figure*}[t!]
    \centering
    \begin{subfigure}[t]{0.3\textwidth}
       \includegraphics[width=\textwidth]{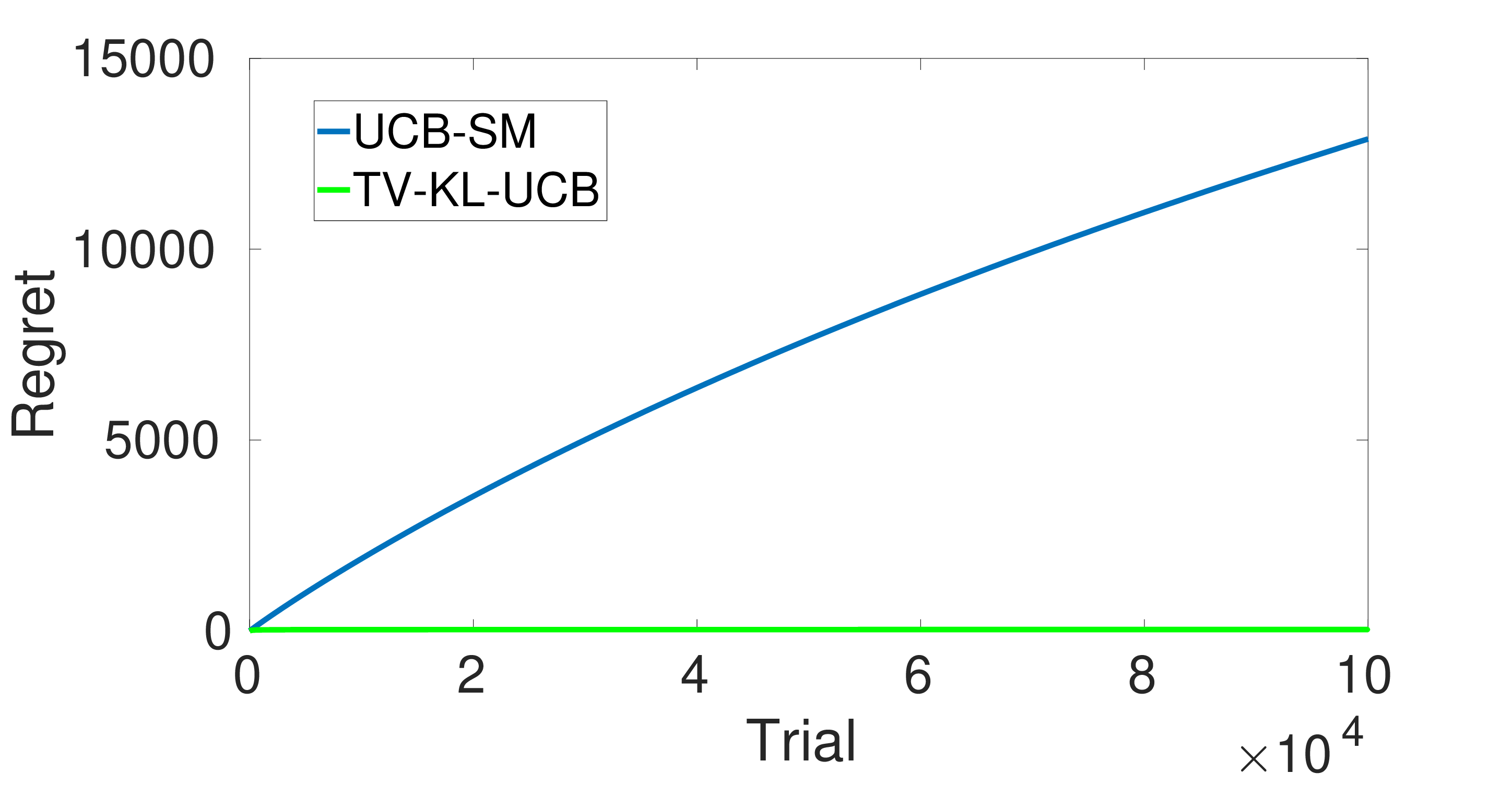}
        \caption{}
        \label{fig:scenario_1_1}
        \end{subfigure}%
    ~
    \begin{subfigure}[t]{0.3\textwidth}
       \includegraphics[width=\textwidth]{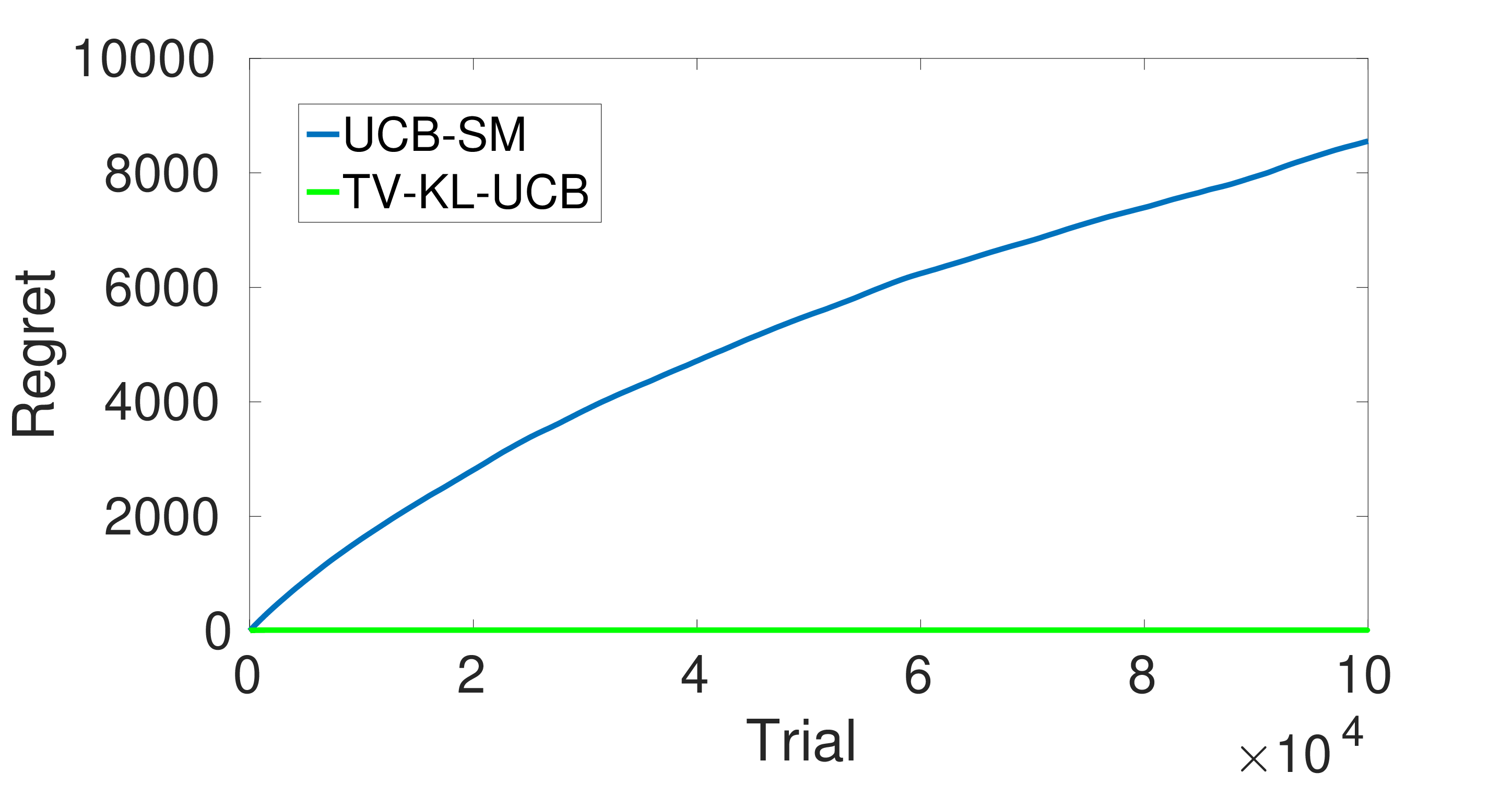}
        \caption{}
        \label{fig:scenario_2_1}
        \end{subfigure}%
    ~
    \begin{subfigure}[t]{0.3\textwidth}
       \includegraphics[width=\textwidth]{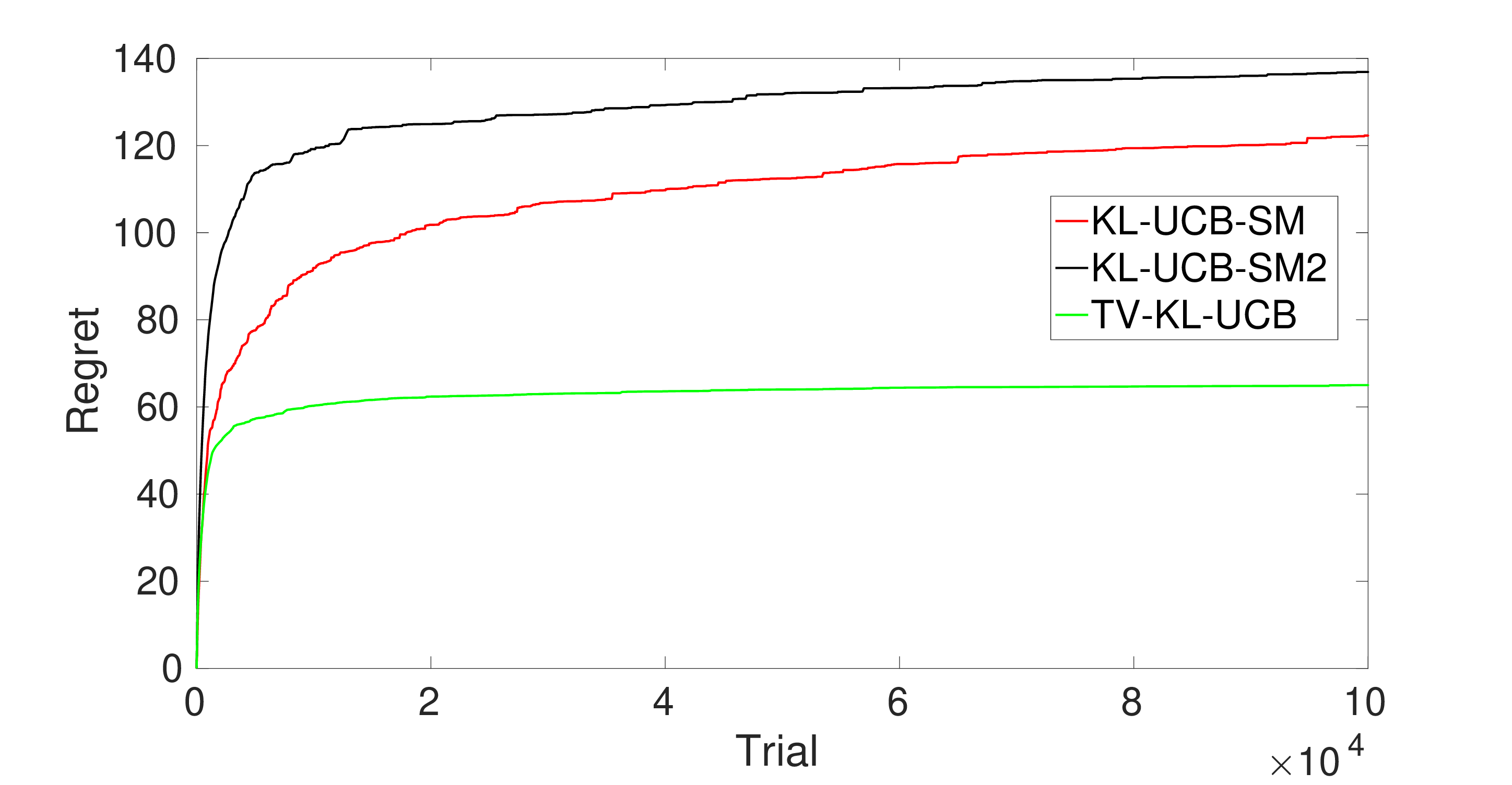}
        \caption{}
        \label{fig:scenario_1_2}
        \end{subfigure}
\\
    \begin{subfigure}[t]{0.3\textwidth}
      \includegraphics[width=\textwidth]{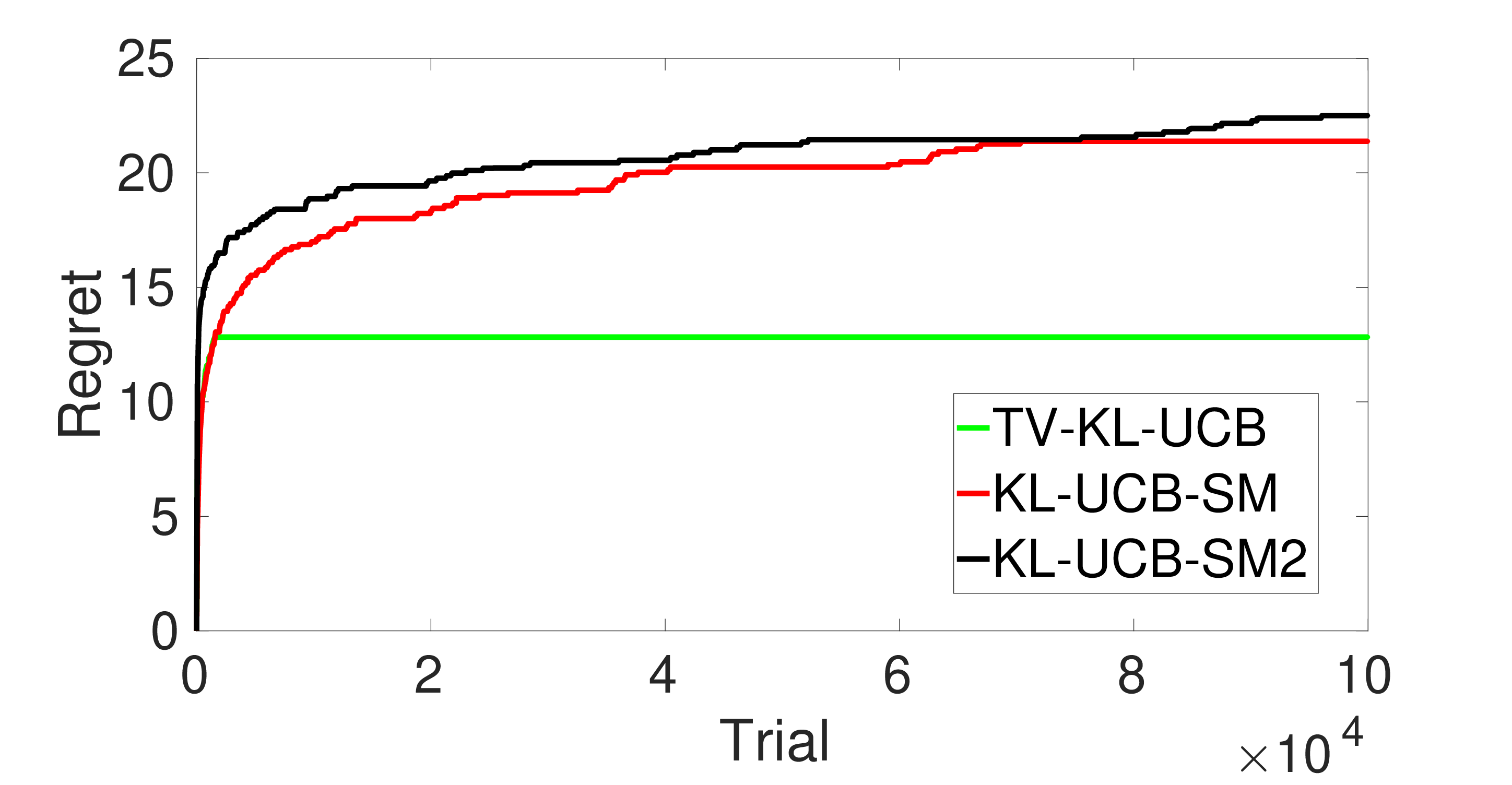}
        \caption{}
        \label{fig:scenario_2_2}
        \end{subfigure}
        ~  \begin{subfigure}[t]{0.3\textwidth}
       \includegraphics[width=\textwidth]{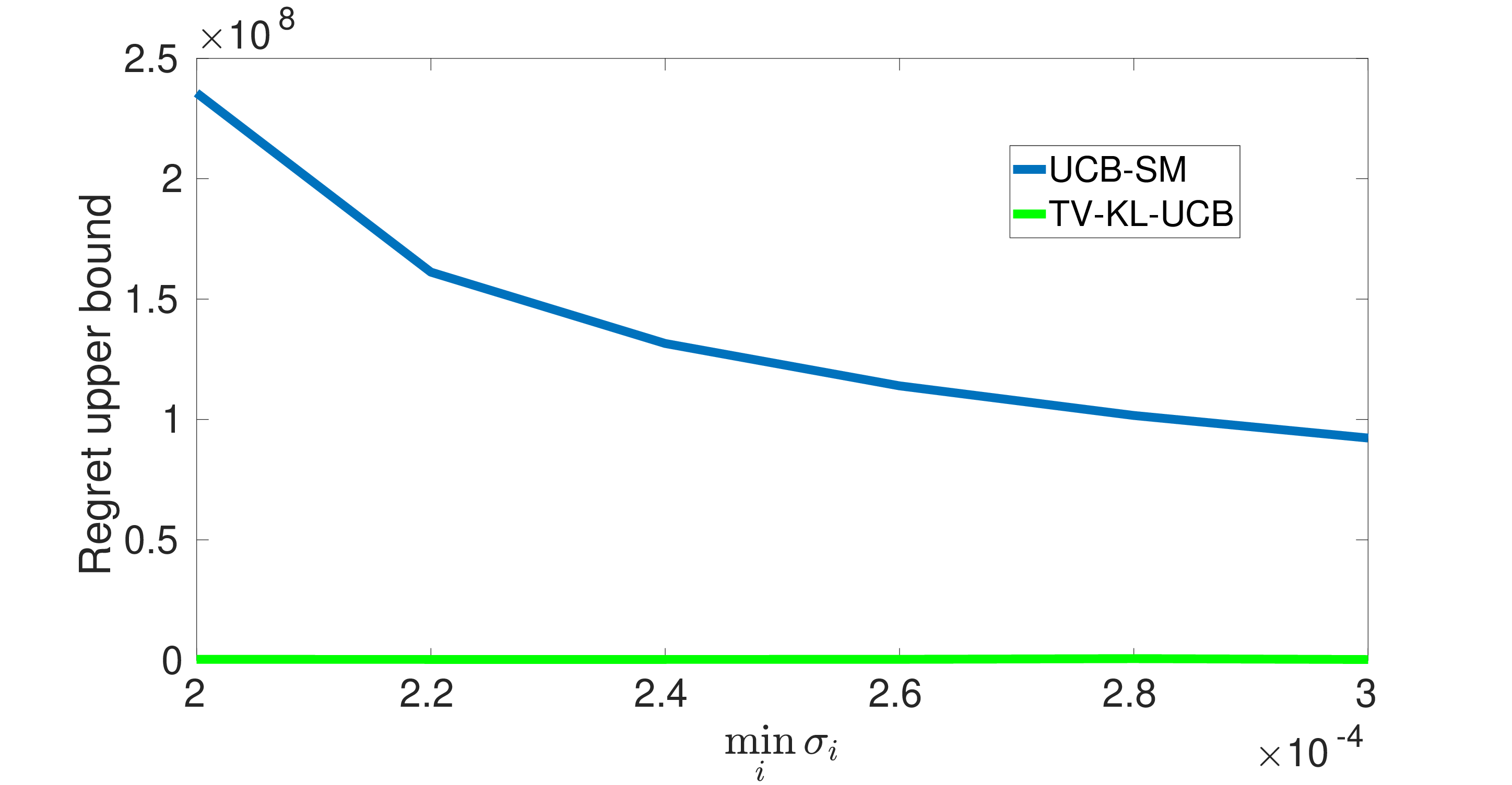}
\caption{}
\label{fig:asymptotic_1_1}
\end{subfigure}
\caption{Comparison of regrets of different algorithms: UCB-SM and \textcolor{black}{TV}-KL-UCB: ((a) scenario 1, (b) scenario 2);
KL-UCB-SM, KL-UCB-SM2 and  \textcolor{black}{TV}-KL-UCB: ((c) scenario 1, (d) scenario 2); (e) regret upper bounds on \textcolor{black}{TV}-KL-UCB and UCB-SM.}
\end{figure*}

\begin{figure*}[t!]
    \centering
    \begin{subfigure}[t]{0.3\textwidth}
       \includegraphics[width=\textwidth]{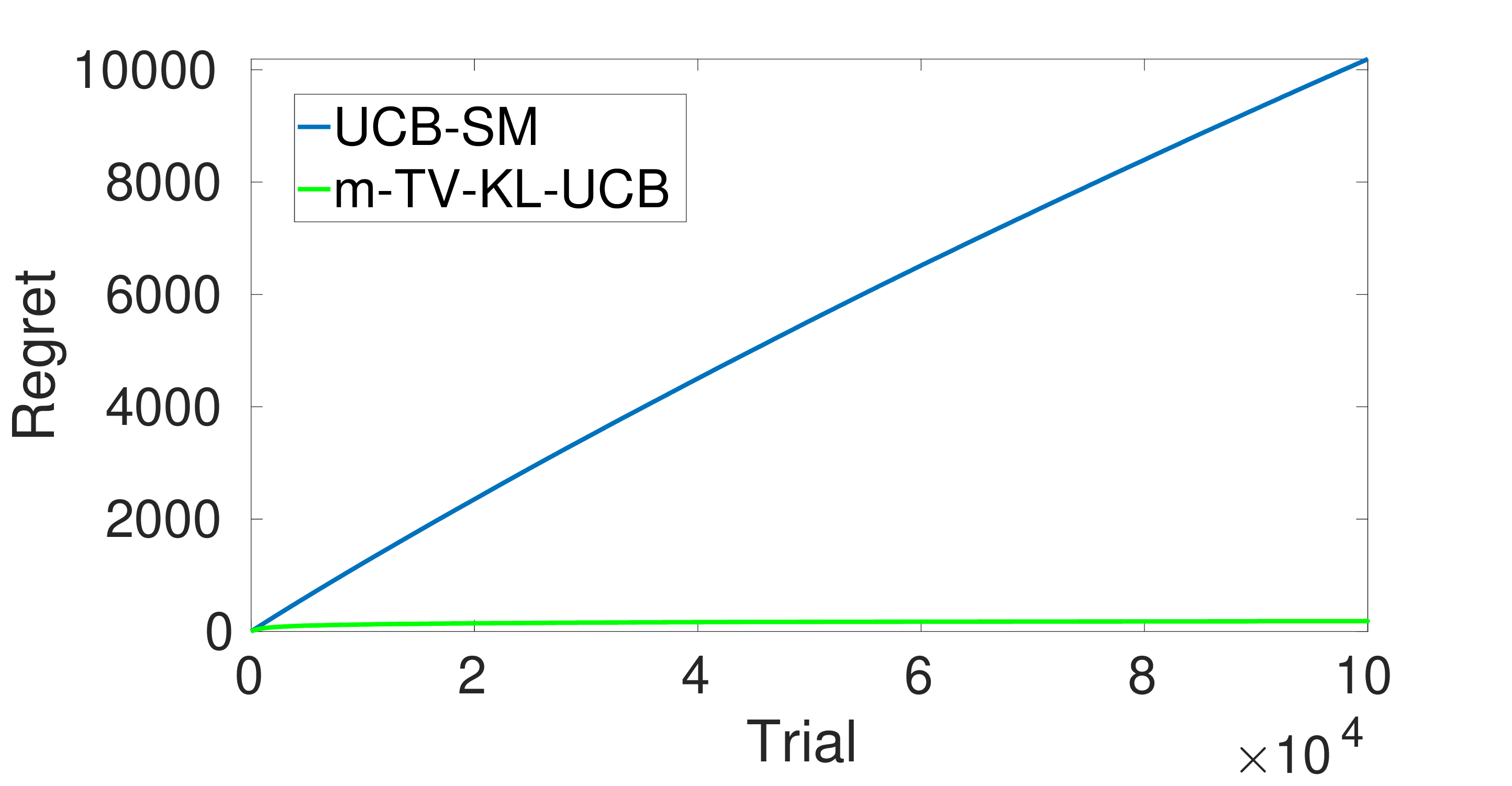}
        \caption{}
        \label{fig:scenario_3_1}
        \end{subfigure}%
    ~
    \begin{subfigure}[t]{0.3\textwidth}
       \includegraphics[width=\textwidth]{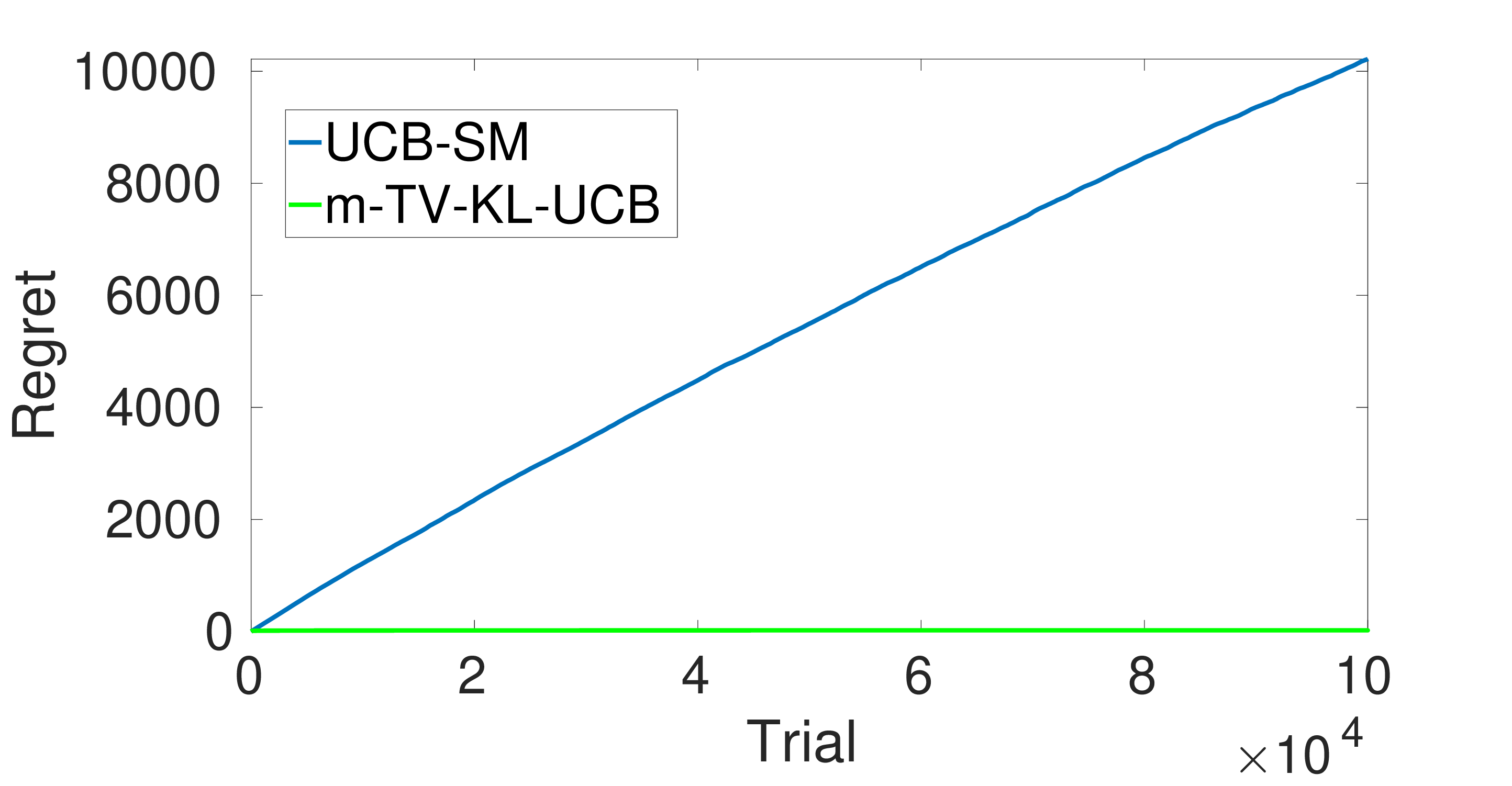}
        \caption{}
        \label{fig:scenario_4_1}
        \end{subfigure}%
    \\
    \begin{subfigure}[t]{0.3\textwidth}
       \includegraphics[width=\textwidth]{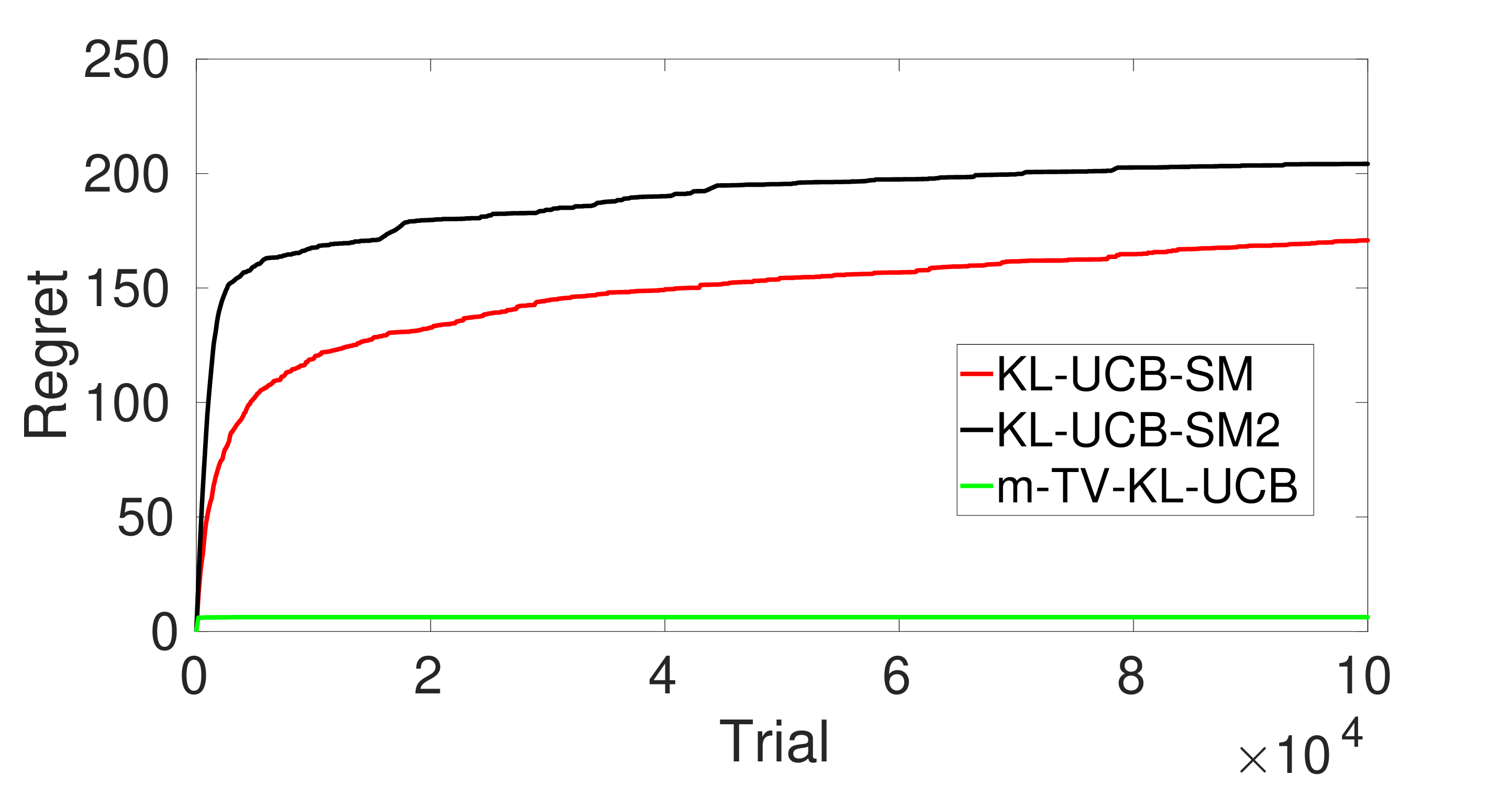}
        \caption{}
        \label{fig:scenario_3_2}
        \end{subfigure}
~
    \begin{subfigure}[t]{0.3\textwidth}
      \includegraphics[width=\textwidth]{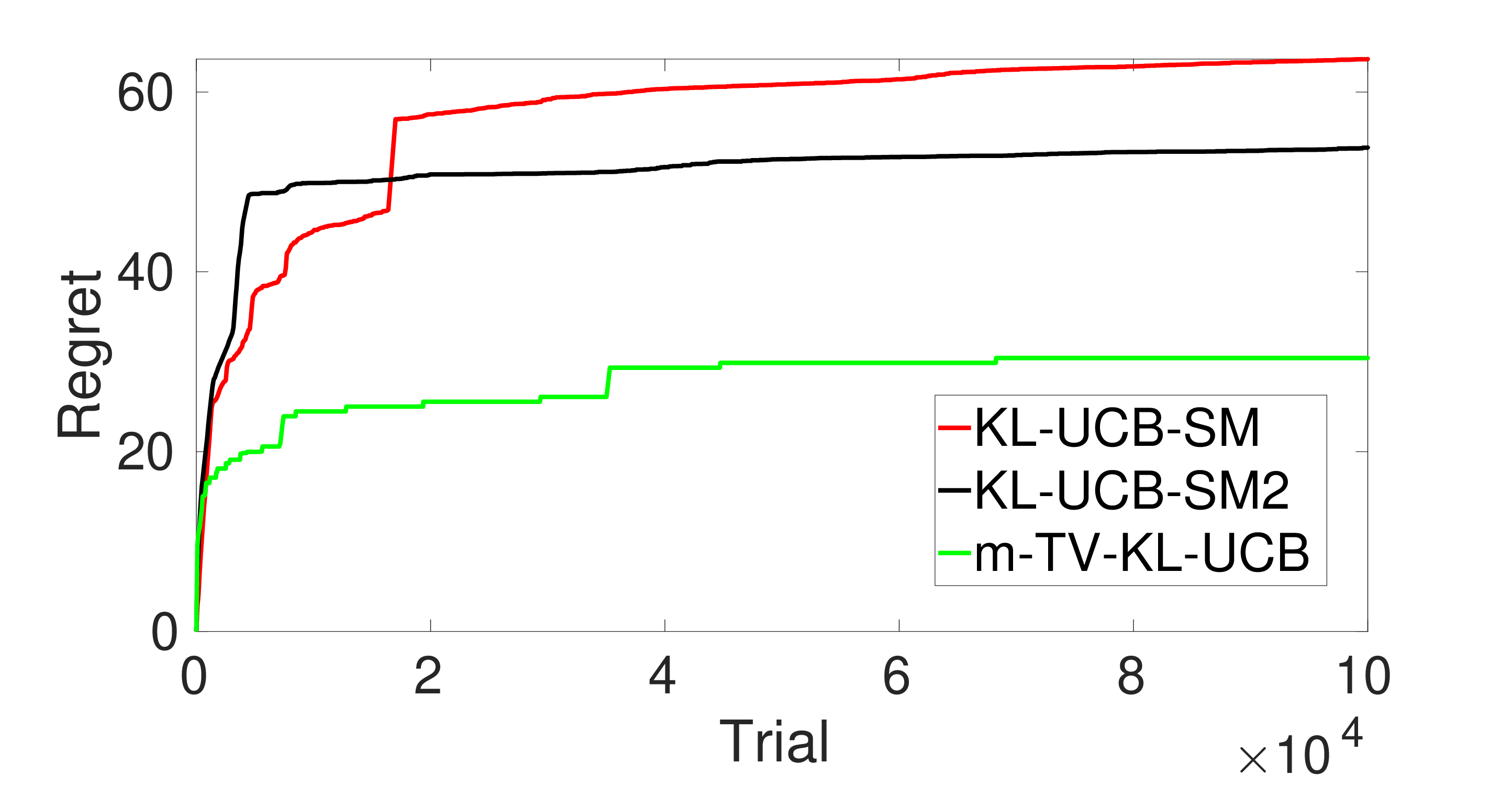}
        \caption{}
        \label{fig:scenario_4_2}
        \end{subfigure}
\caption{Comparison of regrets of different algorithms for scenarios 3 and 4.}
\end{figure*}
In this section, we compare the performances of \textcolor{black}{TV}-KL-UCB and \textcolor{black}{m-TV-KL-UCB} with UCB-SM \cite{tekin2010online}. We also consider an improvement over \cite{tekin2010online}
(referred to as KL-UCB-SM \cite{garivier2011kl}) by replacing the UCB on sample mean by KL-UCB. Arm $A_t$ is chosen at time $t$ using the following scheme:\\
$    A_t=\arg \max \limits_i \max\{\tilde{\mu}:D({\hat{\mu}}^i(t-1)||\tilde{\mu})\le\frac{\log f(t)}{T_i(t-1)}\}.$
KL-UCB-SM2 is a modification of the algorithm in \cite{moulos2020finite} when only one arm can be played at a time.
We consider two scenarios 
and average the result over 100 runs. The code is available at \url{https://github.com/arghyadip89/TV-KL-UCB}. \\ 
    \textbf{Scenario 1:} $p_{01}^1=0.5,p_{01}^2=0.4,p_{01}^3=0.3,p_{01}^4=0.2,p_{01}^5=0.1,
p_{10}^1=0.4,p_{10}^2=0.55,p_{10}^3=0.65,p_{10}^4=0.65,p_{10}^5=0.7$.\par
\textbf{Scenario 2:} $p_{01}^1=0.5,p_{01}^2=0.0004,p_{01}^3=0.0003,p_{01}^4=0.0002,p_{01}^5=0.0001,p_{10}^1=0.4,p_{10}^2=0.00055,p_{10}^3=0.00065,p_{10}^4=0.00065,p_{10}^5=0.0007$.\par
In the first scenario (Fig. \ref{fig:scenario_1_1} and \ref{fig:scenario_1_2}), \textcolor{black}{TV}-KL-UCB significantly outperforms other algorithms. As KL-UCB provides a tighter confidence bound than UCB, the amount of exploration reduces. 
Since \textcolor{black}{TV}-KL-UCB 
identifies the best arm using estimates of $p_{10}^i$, $p_{01}^i$ in a truly Markovian setting, it spends lesser time in exploration than KL-UCB-SM and KL-UCB-SM2 which work based on the sample mean. 
In the second scenario, $p_{01}^i$ and $p_{10}^i$ for $i\neq 1$ are close to zero. As a result, the suboptimal arms spend a considerable amount of time in their initial states. If the initial estimates of $p_{10}^i$ and $p_{01}^i$
are taken to be zero, then the index of a suboptimal arm remain $1$
if it starts at state $0$ (see Equation (\ref{eq:state0})) until it makes a transition to state $1$. Therefore, 
our scheme may lead to a large regret in the beginning until all arms observe a transition to state $1$. 
To address this issue, 
we initialize $\hat{p}_{10}^i$
and $\hat{p}_{01}^i$ to $1$. 
We observe in Fig. \ref{fig:scenario_2_1} and \ref{fig:scenario_2_2} that \textcolor{black}{TV}-KL-UCB performs significantly better than UCB-SM, KL-UCB-SM and KL-UCB-SM2. 
Note that initial values of 
$\hat{p}_{10}^i$
and $\hat{p}_{01}^i$ do not play much role in the first scenario since each arm makes transitions from one state to another comparably often. 
 The asymptotic upper bound on the regret of \textcolor{black}{TV}-KL-UCB is smaller than that of UCB-SM when $\min_i \sigma_i\ge \frac{1}{1440}$ (See Theorem \ref{theo:eigen}). We choose parameters to satisfy $\min_i \sigma_i< \frac{1}{1440}$ and observe in Fig. \ref{fig:asymptotic_1_1} that still the 
asymptotic upper bound on the regret of \textcolor{black}{TV}-KL-UCB is better.
For the i.i.d. case, 
the asymptotic upper bounds on the regrets of both KL-UCB-SM and \textcolor{black}{TV}-KL-UCB match the corresponding lower bound \cite{garivier2011kl}. \par 
We also compare the performance of m-TV-KL-UCB with other algorithms when each arm is a three-state Markov chain with $r(0)=0,r(1)=1/2$ and $r(2)=1$. Similar to Scenarios 1 and 2, we consider two scenarios.\\
    \textbf{Scenario 3:} $p_{01}^1=0.5,p_{01}^2=0.4,p_{01}^3=0.3,p_{01}^4=0.2,p_{01}^5=0.1,
    p_{02}^1=0.1,p_{02}^2=0.2,p_{02}^3=0.3,p_{02}^4=0.2,p_{02}^5=0.7,
p_{12}^1=0.4,p_{12}^2=0.55,p_{12}^3=0.5,p_{12}^4=0.55,p_{12}^5=0.7,
p_{10}^1=0.3,p_{10}^2=0.2,p_{10}^3=0.25,p_{10}^4=0.25,p_{10}^5=0.15,
p_{20}^1=0.6,p_{20}^2=0.5,p_{20}^3=0.4,p_{20}^4=0.3,p_{20}^5=0.2,
p_{21}^1=0.1,p_{21}^2=0.2,p_{21}^3=0.4,p_{21}^4=0.5,p_{21}^5=0.6.$\par
\textbf{Scenario 4:} $p_{01}^1=0.005,p_{01}^2=0.004,p_{01}^3=0.003,p_{01}^4=0.002,p_{01}^5=0.1,
    p_{02}^1=0.001,p_{02}^2=0.002,p_{02}^3=0.003,p_{02}^4=0.002,p_{02}^5=0.7,
p_{12}^1=0.004,p_{12}^2=0.0055,p_{12}^3=0.005,p_{12}^4=0.0055,p_{12}^5=0.7,
p_{10}^1=0.003,p_{10}^2=0.002,p_{10}^3=0.0025,p_{10}^4=0.0025,p_{10}^5=0.15,
p_{20}^1=0.006,p_{20}^2=0.005,p_{20}^3=0.004,p_{20}^4=0.003,p_{20}^5=0.2,
p_{21}^1=0.001,p_{21}^2=0.002,p_{21}^3=0.004,p_{21}^4=0.005,p_{21}^5=0.6.$\par
In both the scenarios, m-TV-KL-UCB performs better than other algorithms as KL-UCB provides a tighter confidence bound than UCB (See Figs \ref{fig:scenario_3_1}, \ref{fig:scenario_4_1},\ref{fig:scenario_3_2}, \ref{fig:scenario_4_2}). Moreover, in truly Markovian settings, as considered by us, m-TV-KL-UCB picks the correct variant of KL-UCB frequently contrary to other algorithms such as KL-UCB-SM and KL-UCB-SM2 which work based on the sample mean. The performance improvement achieved by m-TV-KL-UCB is more than that of TV-KL-UCB as the incorrect variant of KL-UCB leads to more regret when the number of states is more.
\section{Discussions \& Conclusions}\label{sec:discuss}
\textcolor{black}{TV}-KL-UCB uses the \textcolor{black}{total variation} distance between \textcolor{black}{empirical} estimates of $p_{01}^i$ and $(1-p_{10}^i)$ to switch from STP\_PHASE to SM\_PHASE. 
Arm $i$ can be represented uniquely using $p_{01}^i$ and $p_{10}^i$. However, for i.i.d. rewards, it can be represented uniquely using the mean reward. Therefore, usage of STP\_PHASE for i.i.d. arms may lead to a large regret since the additional information (which can be exploited by the sample mean based KL-UCB) is not exploited. We design the algorithm so that it works well for both scenarios. When the arms are i.i.d. (truly Markovian), the SM\_PHASE (STP\_PHASE) is chosen infinitely often. 
Theoretical (under a mild assumption) and experimental results show that the asymptotic upper bound on the regret of TV-KL-UCB is lower than that of UCB-SM \cite{tekin2010online}.
In the i.i.d. case, the upper bounds on the regrets of KL-UCB-SM and \textcolor{black}{TV}-KL-UCB
match the lower bound. 
For \textcolor{black}{TV}-KL-UCB,
to overcome the issue of high regret associated with zero initialization,
we initialize $\hat{p}_{01}$ and $\hat{p}_{10}$ to $1$. Experiments establish that the extension of TV-KL-UCB to a multi-state Markovian model outperforms other algorithms.\par
To conclude, in this paper, we have designed algorithms which achieve a significant improvement over state-of-the-art bandit algorithms. The idea is to detect whether the arm reward is truly Markovian or i.i.d. using a \textcolor{black}{total variation} distance based test. It
switches \textcolor{black}{from} the standard KL-UCB to a sample transition probability based KL-UCB
when it detects that the arm reward is truly Markovian.
Logarithmic upper bounds on the regret of the algorithm have been derived for i.i.d. and Markovian settings. The upper bound on the regret of \textcolor{black}{TV}-KL-UCB matches the lower bound for i.i.d arm rewards. 
\appendices

\section{Lower Bound on Regret of a Uniformly Good Policy}\label{app:a}
We derive a lower bound on the regret associated with any uniformly good policy. It is easy to verify that the lower bound matches with the lower bound in \cite{anantharam1987asymptotically-2}.
Note that unlike \cite{anantharam1987asymptotically-2}, we do not assume that the transition functions belong to a single-parameter family. Therefore, the lower bound derived in this paper is applicable to a larger class of transition functions (class of irreducible Markov chains) than \cite{anantharam1987asymptotically-2}.\par 
\begin{theorem}\label{theo:1}
For a uniformly good policy $\alpha$, 
\begin{equation*}
\liminf \limits_{n \to \infty} \frac{R_{\alpha}(n)}{\log n}\ge \sum\limits_{i=2}^K \frac{\Delta_i}{I(M_i||M_1)}.
\end{equation*}
\end{theorem}
\begin{proof}
We assume that the arm played
and the reward obtained  at time $t$ is denoted by 
$i_t$ and $r_{i_t}$, respectively. 
Let $F_t=\{ i_1,r_{i_1},\ldots, i_{t-1},r_{i_{t-1}}\} $ denote the history of arm selection and reward obtained till time $t$. A policy $\alpha$ is a mapping from the history till time $t$ to the arm played at time $t$ i.e., $\alpha:F_t\to i_t$. 
We consider a pair of arms. Let $\mathcal{M}$ be a set of $0-1$ irreducible Markov chains. Let $\gamma:\mathcal{M}\to \mathcal{R}$ be a function which maps $M \in \mathcal{M}$ to the mean (which is the steady state stationary probability of state $1$ of $M$) of $M$.
Let $\mathcal{E}=\mathcal{M}_1\times \mathcal{M}_2$
where $\mathcal{M}_1$ and $\mathcal{M}_2$ are sets of $0-1$ Markov chains corresponding to arm 1 and arm 2, respectively. Let $E_0=(M_1,M_2) \in \mathcal{E}$ such that $\gamma(M_1)=\mu_1$ and $\gamma(M_2)=\mu_2$, where $\mu_1>\mu_2$. Let $E_1=({M_1}',{M_2}')$ be such that ${M_1}'$ is identical to $M_1$ and $\gamma({M_2}')=\mu_1+\epsilon$,
where $\epsilon>0$.
Note that we consider the two-arm case for the sake of simplicity. The result obtained can be extended easily to $K$ arms by constructing a $K$ dimensional space of Markov chains.\par 
Let $T_{iN}$ be the number of times arm $i$ is played till time $N$. The result will follow if we can show that 
\begin{equation*}
    \liminf_{n \to \infty} \frac{\mathbb{E}_{0}[T_{2N}]}{\log n}\ge \frac{1}{\inf \{I(M_2||{M_2}' ):\gamma({M_2}')>\mu_1\}},
\end{equation*}
where $\mathbb{E}_{0}$ denotes the expectation operator under the probability distribution associated with $E_0$.
Let $\mathbb{P}_{0}$ and $\mathbb{P}_{1}$ denote the probability distributions under $E_0$ and $E_1$, respectively. 
Hence, $D(\mathbb{P}_{0}||\mathbb{P}_{1})=\mathbb{E}_{0}(\log \frac{\mathbb{P}_0(F_N)}{\mathbb{P}_1(F_N)})$, where $F_N$ is the history of $N$ arm pulls. Let $F_{2,N}$  be the part of the history till time $N$ when arm $2$ is played. 
Then, $\frac{\mathbb{P}_0(F_N)}{\mathbb{P}_1(F_N)}=\frac{\mathbb{P}_0(F_{2,N})}{\mathbb{P}_1(F_{2,N})}$
(since arm $1$ is identical in both the cases, and probabilities associated with $\alpha$ (if randomized) also cancel in the numerator and the denominator).
Let $t_{ij}^{(N)}$ be the number of times arm $2$ transitions from state $i\in \mathcal{S}$ to state $j\in \mathcal{S}$ in $F_{2,N}$. Let $p_{ij}(k)$ be the probability of transition from state $i$ to state $j$ of arm $2$ under $E_k$. Then,
\begin{equation*}
\frac{\mathbb{P}_0(F_N)}{\mathbb{P}_1(F_N)}=\prod_{(i,j)\in \mathcal{S}}{\big(\frac{p_{ij}(0)}{p_{ij}(1)}\big)}^{t_{ij}^{(N)}}.    
\end{equation*}
Hence, 
\begin{equation*}
    \log \frac{\mathbb{P}_0(F_N)}{\mathbb{P}_1(F_N)}=\sum_{(i,j)\in \mathcal{S}} t_{ij}^{(N)} \log \frac{p_{ij}(0)}{p_{ij}(1)},
\end{equation*}
and 
\begin{equation*}
    D(\mathbb{P}_0||\mathbb{P}_1)=\sum_{(i,j)\in \mathcal{S}} \mathbb{E}_0(t_{ij}^{(N)}) \log \frac{p_{ij}(0)}{p_{ij}(1)}.
\end{equation*}

We know that for any event $A$ \cite[Chapter~14]{lattimore2018bandit},
\begin{equation*}
    \mathbb{P}_0(A)+\mathbb{P}_1(A^{\mathrm{C}})\ge \frac{1}{2} \exp(-D(\mathbb{P}_0||\mathbb{P}_1)).
\end{equation*}

Therefore, for any event $A$,
\begin{equation*}
    \mathbb{P}_0(A)+\mathbb{P}_1(A^{\mathrm{C}})\ge \frac{1}{2}\exp\big(-\sum_{(i,j)\in \mathcal{S}} \mathbb{E}_0(t_{ij}^{(N)}) \log \frac{p_{ij}(0)}{p_{ij}(1)}\big).
\end{equation*}
Choose $A=\{T_{1N}\le T_{2N}\}$.
We have,
\begin{equation}\label{eq:markov_1}
\begin{split}
\mathbb{P}_0(A)&=\mathbb{P}_0(T_{1N}\le T_{2N})=\mathbb{P}_0(T_{2N}\ge \frac{N}{2})\\&  
\le \frac{2}{N}\mathbb{E}_0(T_{2N}) =\frac{2}{N}o(N^\beta),
\end{split} 
\end{equation}
where the second and last equalities follow from $T_{1N}+T_{2N}=N$ and the property of uniformly good policy, respectively. The inequality
is a direct application of Markov's inequality.
Similarly, 
\begin{equation}\label{eq:markov_2}
\begin{split}
\mathbb{P}_1(A^{\mathrm{C}})&=\mathbb{P}_1(T_{1N}\ge T_{2N})=\mathbb{P}_1(T_{1N}\ge \frac{N}{2})   \\&
\le \frac{2}{N}\mathbb{E}_1(T_{1N})=\frac{2}{N}o(N^\beta).
\end{split} 
\end{equation}

Thus, if the policy is uniformly good, then (using Equations (\ref{eq:markov_1}) and (\ref{eq:markov_2})),
\begin{equation*}
\frac{2}{N}o(N^\beta)\ge \frac{1}{2}\exp\big(-\sum_{(i,j)\in \mathcal{S}} \mathbb{E}_0(t_{ij}^{(N)}) \log \frac{p_{ij}(0)}{p_{ij}(1)}\big).
\end{equation*}
Hence, 
\begin{equation*}
\sum_{(i,j)\in \mathcal{S}} \mathbb{E}_0(t_{ij}^{(N)}) \log \frac{p_{ij}(0)}{p_{ij}(1)}\ge \log N-\log o(N^{\beta})-\log 4.    
\end{equation*}
Taking appropriate limits we obtain,
\begin{equation*}
    \liminf_{N\to\infty}\frac{1}{\log N}\sum_{(i,j)\in \mathcal{S}} \mathbb{E}_0(t_{ij}^{(N)}) \log \frac{p_{ij}(0)}{p_{ij}(1)}\ge 1.
\end{equation*}
Let $\tau_i^{(N)}$ be the number of times arm $2$ is chosen while it was in state $i \in \mathcal{S}$ in the first $(N-1)$ plays in the history. Then,
\begin{equation*}
    \mathbb{E}_0(t_{ij}^{(N)})=p_{ij}(0)\mathbb{E}_0(\tau_i^{(N)}).
\end{equation*}
Now, 
\begin{equation*}
\sum_{(i,j)\in \mathcal{S}} \mathbb{E}_0(t_{ij}^{(N)}) \log \frac{p_{ij}(0)}{p_{ij}(1)}=
\sum_{i\in \mathcal{S}} \mathbb{E}_0({\tau}_{i}^{(N)}) D(p_{i0}(0)||p_{i0}(1)).
\end{equation*}
Hence,
\begin{equation*}
    \liminf_{N\to\infty}\frac{1}{\log N}\sum_{i\in \mathcal{S}} \mathbb{E}_0({\tau}_{i}^{(N)}) D(p_{i0}(0)||p_{i0}(1))\ge 1.
\end{equation*}
 We define the regret as 
$R_{\alpha}(n)=(\mu_1-\mu_2)(\mathbb{E}_0({\tau}_{0}^{(N)})+\mathbb{E}_0({\tau}_{1}^{(N)}))$
since the regret depends on the number of times arm $2$ is chosen.\par 
An asymptotic lower bound on the regret can be obtained by solving the following optimization problem.
\begin{equation*}
\begin{split}
& \min_{\epsilon>0} \min_{p_{.0}(1),p_{.1}(1)}(\mu_1-\mu_2)(x+y)\\&   
\text{such that}\\&
xD(p_{00}(0)||p_{00}(1))+yD(p_{10}(0)||p_{10}(1))\ge 1,\\&
p_{01}(1)(\mu_1+\epsilon)=p_{10}(1)(1-\mu_1-\epsilon),\\&
p_{00}(1)+p_{01}(1)=1,\\&
p_{10}(1)+p_{11}(1)=1,\\&
\text{and } p_{ij}(1)\ge 0 \ \forall i,j.
 \end{split}
\end{equation*}
Dividing the first constraint by $(x+y)$ on both sides we get,
\begin{equation*}
    (x+y)\ge \frac{1}{\frac{x}{x+y}D(p_{00}(0)||p_{00}(1))+\frac{y}{x+y}D(p_{10}(0)||p_{10}(1))}.
\end{equation*}
Since $\frac{x}{x+y}=\pi_2(0)$ and
$\frac{y}{x+y}=\pi_2(1)$, we have,
\begin{equation*}
    (x+y)\ge \frac{1}{\pi_2(0)D(p_{00}(0)||p_{00}(1))+\pi_2(1)D(p_{10}(0)||p_{10}(1))}.
\end{equation*}
Instead of perturbing the mean of arm $1$ (with transition probability functions $p_{01}^1$ and $p_{10}^1$) by $\epsilon$, we can consider a new Markov chain with transition probabilities ${p'}_{01}^1$ and ${p'}_{10}^1$.
For fixed values of these parameters, we get 
\begin{equation*}
    (x+y)\ge \frac{1}{\pi_2(0)D(p_{01}^2||{p'}_{01}^1)+\pi_2(1)D(p_{10}^2||{p'}_{10}^1)}.
\end{equation*}
Therefore, the optimization problem reduces to the following problem 
\begin{equation*}
\inf_{{p'}_{10}^1,{p'}_{01}^1}\frac{1}{\pi_2(0)D(p_{01}^2||{p'}_{01}^1)+\pi_2(1)D(p_{10}^2||{p'}_{10}^1)}    
\end{equation*}
such that the stationary probability of state $1$ under the transition law $({p'}_{01}^1,{p'}_{10}^1)$ is greater than that under $({p}_{01}^1,{p}_{10}^1)$.\par 
Assume that the constraint becomes an equality under the optimal solution.
Then, we have,
\begin{equation*}
\begin{split}
&R_{\alpha}(n)\ge \frac{(\mu_1-\mu_2)}{\pi_2(0)D(p_{01}^2||p_{01}^1)+\pi_2(1)D(p_{10}^2||p_{10}^1)}\\&=\frac{(\mu_1-\mu_2)}{I(M_2||M_1)}.    
\end{split}
\end{equation*}
This completes the proof of the theorem.
\end{proof}

\section{Proof of Theorem \ref{theo:regret_2}}\label{app:theo}
We describe a set of useful results before deriving the asymptotic upper bound on the regret. 
\begin{proposition}\label{lemma:pinsker}
Let $p,q,\epsilon \in[0,1]$. The following relations hold:\\
(a) $D(p||q)\ge 2(p-q)^2$ (Pinsker's inequality),\\
(b) If $p\le q-\epsilon \le q$, then 
$D(p||q-\epsilon)\le D(p||q)-2\epsilon^2$,\\
(c) If $p\le p+\epsilon\le q$, then 
$D(q||p+\epsilon) \le D(q||p)-2\epsilon^2$.
\end{proposition}
\begin{proof}
Proofs of (a) and (b) are given in \cite[Lemma~10.2]{lattimore2018bandit}. 
Let $h(q)=D(q||p+\epsilon)-D(q||p)$. 
We get, $h'(q)=\log \frac{p(1-p-\epsilon)}{(p+\epsilon)(1-p)}<0$ (since $p\le p+\epsilon\le q$). Therefore, $h(q)$ is linear and decreasing in $q$. Thus, 
   $ h(q)\le h(p+\epsilon)=-D(p+\epsilon||p)\le -2\epsilon^2$.
\end{proof}
\begin{corollary}\cite[Corollary~10.4]{lattimore2018bandit}\label{corollary:1}
For any $a\ge0$,
$\mathbb{P}(D(\hat{\mu}||\mu)\ge a, \hat{\mu}\le \mu)\le \exp(-na)$
and 
$\mathbb{P}(D(\hat{\mu}||\mu)\ge a, \hat{\mu}\ge \mu)\le \exp(-na).$
\end{corollary}
Let the transition probabilities from state 0 to state 1 and state 1 to state 0 of a two-state Markov chain
$\{M_i\}_{i\ge 0}$ be $P_{01}$ and $P_{10}$, respectively. 
Let $P_{00}=1-P_{01}$ and 
$P_{11}=1-P_{10}$.
Let the stationary probabilities of states 0 and 1 be $\pi_0$ and $\pi_1$, respectively. 
Let $N_0(t)=\sum_{i=1}^{t}\mathbbm{1}\{M_i=0\}$,
$N_1(t)=\sum_{i=1}^{t}\mathbbm{1}\{M_i=1\}$ and $N_{ij}(t)$ denote the number of times the Markov chain transitions from state $i$ to state $j$
till time $t$. Let 
$\hat{P}_{ij}(t)=\frac{N_{ij}(t)}{N_i(t)}$. 
The proposition presented next establishes that the fraction of visits to any state of a Markov chain is never too away from the stationary probability of the state.
\textcolor{black}{This describes the contribution of mixing time to the regret upper bound. Clearly the upper bound in Proposition \ref{lemma:mixing} is finite and does not contribute to the asymptotic regret bound.} The proposition described next depicts that the estimates of the transition probabilities associated with a Markov chain are never too far from the true transition probabilities.
Proposition \ref{lemma:Hellinger}, 
alongside Borel-Cantelli Lemma, establishes that  appropriate conditions on the \textcolor{black}{total variation} distance are satisfied infinitely often for i.i.d. ($P_{01}+P_{10}=1$) and truly Markovian arm reward ($P_{01}+P_{10}\neq 1$), respectively.

\begin{proposition}\label{lemma:mixing}
Let $C_t:=\{|\frac{N_0(t)}{t-1}-\pi_0|>\epsilon_1\}$. 
Then, 
\begin{equation*}
\sum_{t=1}^{n}\mathbb{P}(C_t)\le \sum_{t=1}^{\infty} (t+1)^3  \exp(-2(t-1){\epsilon_1}^2(P_{01}+P_{10})^2).     
\end{equation*}
\end{proposition}
\begin{proposition}\label{lemma:concentratation}
Let $D_{0,t}:=\{|\hat{P}_{01}(t)-P_{01}|>\epsilon_1\}$
and 
 $D_{1,t}:=\{|\hat{P}_{10}(t)-P_{10}|>\epsilon_1\}$.
Then, 
\begin{equation*}
\sum_{t=1}^{n}\mathbb{P}(D_{0,t})\le \frac{1}{{\epsilon_1}^2}+ \sum_{t=1}^{\infty} (t+1)^3  \exp(-2(t-1){\epsilon_1}^2(P_{01}+P_{10})^2),    
\end{equation*}
\begin{equation*}
\sum_{t=1}^{n}\mathbb{P}(D_{1,t})\le \frac{1}{{\epsilon_1}^2}+ \sum_{t=1}^{\infty} (t+1)^3  \exp(-2(t-1){\epsilon_1}^2(P_{01}+P_{10})^2).
\end{equation*}
\end{proposition}
\begin{proposition}\label{lemma:Hellinger}
Let $B_t:=\{\textcolor{black}{(|{\hat{p}}_{01}^i(t)+{\hat{p}}_{10}^i(t)-1|<\frac{1}{t^{1/4}}\}}$. 
If $P_{01}+P_{10}=1$, then 
\begin{equation*}
\sum_{t=1}^n \mathbb{P}(B_t^c) \le \sum_{t=1}^{\infty} 4 \exp(-\frac{2}{9}\sqrt{t})+\sum_{t=1}^{\infty} (t+1)^3  \exp(-2\sqrt{t-1}).
\end{equation*}
If $P_{01}+P_{10}< 1$, then for $\tau\ge \frac{1}{\textcolor{black}{|P_{01}+{P}_{10}+2\epsilon_1-1|^4}}$ and if $P_{01}+P_{10}> 1$, then for $\tau\ge \frac{1}{\textcolor{black}{|P_{01}+{P}_{10}-2\epsilon_1-1|^4}}$,\\ 
\begin{equation*}
\sum_{t=\tau}^n \mathbb{P}(B_t) \le \frac{2}{{\epsilon_1}^2}+ \sum_{t=1}^{\infty} 2(t+1)^3  \exp(-2(t-1){\epsilon_1}^2(P_{01}+P_{10})^2),
\end{equation*}
where $\epsilon_1<\frac{|P_{01}+P_{10}-1|}{2}$.
\end{proposition}
Proofs of Propositions \ref{lemma:mixing}, \ref{lemma:concentratation} and \ref{lemma:Hellinger} are provided in Appendix \ref{app:234}.
The proposition presented next is used to prove that the confidence bounds on transition probabilities associated with the optimal arm are never too far from the respective true transition probabilities. The  proposition thereafter is used to establish that the index associated with a sub-optimal arm is not often much greater than the index of the optimal arm.
Let $p_1,p_2,q_1,q_2\in[0,1]$.
Let $Z_1(s)$ be a non-negative random variable with $Z_1(s)\in[0,1], s=1,2,\ldots,n$. Let $W_{1},W_{2},\ldots, W_{n}$ be i.i.d. Bernoulli random variables with mean $\mu \in [0,1]$, and $Z_2(s)$ is measurable with respect to $\sigma\{W_{1},\ldots,W_{s}\}$.
Proofs of Propositions \ref{lemma_H:1_Markovian} and \ref{lemma:2a} are given in
Appendix \ref{app:234}.

  \begin{proposition}\label{lemma_H:1_Markovian}
(a) Let $X_1,X_2,\ldots, X_n$ be i.i.d. Bernoulli random variables with mean $p\in[0,1]$, and $\epsilon_{p}>0$. 
Let $\hat{p}_s=\frac{1}{s}\sum \limits_{i=1}^{s}X_s$, where $s$ is the number of samples till time $t$.
Let $\tilde{p}_s^*= \max\{\tilde{p}\in [\hat{p}_s,1]:D(\hat{p}_s||\tilde{p})\le \frac{\log f(t)}{s}\}.$
Then,
$\sum_{s=1}^{n}\mathbb{P}(\tilde{p}_s^*<p-\epsilon_{p})
\le \frac{2}{{\epsilon_{p}}^2}$.
\\(b) Let $Y_1,Y_2,\ldots, Y_n$ be i.i.d. Bernoulli random variables with mean $q\in[0,1]$, $\epsilon_{q}>0$. 
Let $\hat{q}_s=\frac{1}{s}\sum \limits_{i=1}^{s}Y_s$  
and $\tilde{q}_s^*= \min\{\tilde{q}\in [0,\hat{q}_s]:D(\hat{q}_s||\tilde{q})\le \frac{\log f(t)}{s}\},$
Then,
$\sum_{s=1}^{n}\mathbb{P}(\tilde{q}_s^*>q+\epsilon_{q})
\le \frac{2}{{\epsilon_{q}}^2}$.
\end{proposition}

\begin{proposition}\label{lemma:2a}
(a)        $\kappa_1:=\sum\limits_{s=c}^n \mathbb{P}(D(p_2+\epsilon_1||\frac{(p_1-\epsilon_p)(q_2-\epsilon_1)}{q_1+\epsilon_1}+Z_1(s))\le \frac{a}{s})$.
Then 
$    \kappa_1\le \inf\limits_{\epsilon_1}(\frac{a}{D(p_2+\epsilon_1||\frac{(p_1-\epsilon_p)(q_2-\epsilon_1)}{q_1+\epsilon_1})}-(c-1))^+$,\\
(b) 
$   \kappa_2:=\sum\limits_{s=c}^n \mathbb{P}(D(q_2-\epsilon_1||\frac{(q_1+\epsilon_1)(p_2+\epsilon_1)}{p_1-\epsilon_p}-Z_1(s))\le \frac{a}{s})$.
Then 
    $\kappa_2\le \inf\limits_{\epsilon_1}(\frac{a}{D(q_2-\epsilon_1||\frac{(q_1+\epsilon_1)(p_2+\epsilon_1)}{p_1-\epsilon_p})}-(c-1))^+$,
\\
(c) 
    $\kappa_3:=\sum\limits_{s=c}^{n} \mathbb{P}(D(\hat{\mu}_s||\frac{p-\epsilon_p}{p-\epsilon_p+q+\epsilon_1}+Z_2(s))\le \frac{a}{s}).$
Then 
    $\kappa_3\le \inf\limits_{\epsilon_1}[\frac{1}{2{\epsilon_1}^2}+(\frac{a}{D(\mu+\epsilon_1||\frac{p-\epsilon_{p}}{p-\epsilon_{p}+q+\epsilon_1})}-(c-1))^+]$,
where $\mu+\epsilon_1<\frac{p-\epsilon_p}{p-\epsilon_p+q+\epsilon_1}$ and  
$p_2+\epsilon_1<\frac{(p_1-\epsilon_p)(q_2-\epsilon_1)}{q_1+\epsilon_1}$.
\end{proposition}
\subsubsection{Truly Markovian Arms}
An arm can either be in STP\_PHASE or in SM\_PHASE at any time $t$ depending on whether the condition on the \textcolor{black}{total variation} distance is satisfied or not. When all arms are truly Markovian, we establish that the expected number of times the arms are in SM\_PHASE, is finite. In other words, for both optimal and suboptimal arms, the appropriate conditions on the \textcolor{black}{total variation} distances 
are satisfied infinitely often.
Let $E_{1,i,t}:=\{\textcolor{black}{|{\hat{p}}_{01}^1(t)-1+{\hat{p}}_{10}^1(t)|>\frac{1}{{t}^{1/4}}}$,$\textcolor{black}{{|\hat{p}}_{01}^i(t)-1+{\hat{p}}_{10}^i(t)|>\frac{1}{{t}^{1/4}}}\}$. 
 For $t_{1,i}=
%
\frac{\mathbbm{1}\{p_{01}^i+p_{10}^i<1\}}{\textcolor{black}{TV^4}(p_{01}^i+\epsilon_1||1-{p}_{10}^i-\epsilon_1)}+
\frac{\mathbbm{1}\{p_{01}^i+p_{10}^i>1\}}{\textcolor{black}{TV^4}(p_{01}^i-\epsilon_1||1-{p}_{10}^i+\epsilon_1)}+
\frac{\mathbbm{1}\{p_{01}^1+p_{10}^1<1\}}{\textcolor{black}{TV^4}(p_{01}^1+\epsilon_1||1-{p}_{10}^1-\epsilon_1)}+
\frac{\mathbbm{1}\{p_{01}^1+p_{10}^1>1\}}{\textcolor{black}{TV^4}(p_{01}^1-\epsilon_1||1-{p}_{10}^1+\epsilon_1)}$,    
we obtain (using Proposition \ref{lemma:Hellinger})
\begin{equation}\label{eq:e_1_tc}
\begin{aligned}
&    \sum_{t=t_{1,i}}^n\mathbb{P}(E_{1,i,t}^c)
\le \sum_{t=t_{1,i}-1}^{n-1}\mathbb{P}(\textcolor{black}{TV}({\hat{p}}_{01}^1(t)||1-{\hat{p}}_{10}^1(t))\le\frac{1}{t^{1/4}})\\&+\mathbb{P}(\textcolor{black}{TV}({\hat{p}}_{01}^i(t)||1-{\hat{p}}_{10}^i(t))\le\frac{1}{t^{1/4}})\\&
\le \frac{4}{{\epsilon_1}^2}+2 \sum_{t=1}^{\infty} (t+1)^3  \exp(-2(t-1){\epsilon_1}^2(p_{01}^i+p_{10}^i)^2)\\&+
2 \sum_{t=1}^{\infty} (t+1)^3  \exp(-2(t-1){\epsilon_1}^2(p_{01}^1+p_{10}^1)^2).
    \end{aligned}
\end{equation}
Next, we show that the estimates of transition probabilities of optimal and suboptimal arms are never too far from respective true values.
$F_{1,i,t}:=\{|{{\hat{p}}_{01}^1}(t)-p_{01}^1|\le \epsilon_1,|{{\hat{p}}_{10}^1}(t)-p_{10}^1|\le \epsilon_1,|{{\hat{p}}_{01}^i}(t)-p_{01}^i|\le\epsilon_1,|{{\hat{p}}_{10}^i}(t)-p_{10}^i|\le\epsilon_1\}.$  
Similar to (\ref{eq:e_1_tc}) and using Proposition \ref{lemma:concentratation},
\begin{equation}\label{eq:f_1_t_c}
\begin{aligned}
&    \sum_{t=1}^n\mathbb{P}(F_{1,i,t}^c)
\\& \le \frac{4}{{\epsilon_1}^2}+
2\sum_{t=1}^{\infty} (t+1)^3  \exp(-2(t-1){\epsilon_1}^2(p_{01}^i+p_{10}^i)^2)\\&
+2\sum_{t=1}^{\infty} (t+1)^3  \exp(-2(t-1){\epsilon_1}^2(p_{01}^1+p_{10}^1)^2).
     \end{aligned}
\end{equation}
 Let the number of times sub-optimal arm $i$ is pulled till time $n$ and the state of arm $i$ at time $t$ be denoted by $T_i(n)$ and
 $S_i(t)$, respectively. \textcolor{black}{Let $E^{a,b}_{1,i,t}:=\{E_{1,i,t},S_i(t)=a,S_1(t)=b\}.$} 

 \begin{equation}\label{eq:regret_H_Markovian}
 \begin{aligned}
    & \mathbb{E}[T_i(n)]
    =\sum_{t=1}^{\tau_{1,i}}\mathbb{P}(A_t=i)+\sum_{t=\tau_{1,i}+1}^n\mathbb{P}(A_t=i)\\&
    \le \tau_{1,i} + \sum_{t=\tau_{1,i}+1}^n\mathbb{P}(E_{1,i,t}^c)+\sum_{t=\tau_{1,i}+1}^n\mathbb{P}(A_t=i,E_{1,i,t}).
 \end{aligned}
 \end{equation}
 Now, we consider the last term in (\ref{eq:regret_H_Markovian}). 
When $E_{1,i,t}$ occurs, 
sub-optimal arm $i$ is chosen if at least one of the following conditions is true.
 \begin{enumerate}
     \item ${\tilde{p}}_{01}^{*1}(t)<p_{01}^1-\epsilon_p$ and $S_1(t)=0$,
    \item ${\tilde{p}}_{10}^{*1}(t)>p_{10}^1+\epsilon_q$ and  $S_1(t)=1$,
     \item $\frac{p_{01}^1-\epsilon_p}{p_{01}^1-\epsilon_p+{{\hat{p}}_{10}^1}(t)}<\frac{{\tilde{p}}_{01}^{*i}(t)}{{{\tilde{p}}_{01}}^{*i}(t)+{{\hat{p}}_{10}^i}(t)}$, $S_i(t)=0$ and $S_1(t)=0$,
     \item $\frac{p_{01}^1-\epsilon_p}{p_{01}^1-\epsilon_p+{{\hat{p}}_{10}^1}(t)}<\frac{{\hat{p}}_{01}^{i}(t)}{{\hat{p}}_{01}^{i}(t)+{{\tilde{p}}_{10}}^{*i}(t)}$, $S_i(t)=1$ and $S_1(t)=0$,
     \item $\frac{{{\hat{p}}_{01}^1}(t)}{{{\hat{p}}_{01}^1}(t)+p_{10}^1+\epsilon_q}<\frac{{\tilde{p}}_{01}^{*i}(t)}{{{\tilde{p}}_{01}}^{*i}(t)+{{\hat{p}}_{10}^i}(t)}$, $S_i(t)=0$ and $S_1(t)=1$,
     \item $\frac{{{\hat{p}}_{01}^1}(t)}{{{\hat{p}}_{01}^1}(t)+p_{10}^1+\epsilon_q}<\frac{{\hat{p}}_{01}^{i}(t)}{{\hat{p}}_{01}^{i}(t)+{{\tilde{p}}_{10}}^{*i}(t)}$, $S_i(t)=1$ and $S_1(t)=1$.
 \end{enumerate}
 \begin{equation}\label{eq:regret_H_2_Markovian}
 \begin{aligned}
&  \sum_{t=\tau_{1,i}+1}^n   \mathbb{P}(A_t=i,E_{1,i,t})\\&\le \sum_{t=\tau_{1,i}+1}^n \mathbb{P}(A_t=i,{\tilde{p}}_{01}^{*1}(t)<p_{01}^1-\epsilon_p,S_1(t)=0,E_{1,i,t})\\&
+ \mathbb{P}(A_t=i,{\tilde{p}}_{10}^{*1}(t)>p_{10}^1+\epsilon_q
,S_1(t)=1,E_{1,i,t})\\&
+     \mathbb{P}(A_t=i,{{\tilde{p}}_{01}}^{*i}(t)>\frac{(p_{01}^1-\epsilon_{p}) {{\hat{p}}_{10}^i}(t)}{{{\hat{p}}_{10}^1}(t)},\textcolor{black}{E^{0,0}_{1,i,t}})\\&
+     \mathbb{P}(A_t=i,{{\tilde{p}}_{10}}^{*i}(t)<\frac{{{\hat{p}}_{01}^i}(t){{\hat{p}}_{10}^1}(t)}{(p_{01}^1-\epsilon_{p})},\textcolor{black}{E^{1,0}_{1,i,t}})\\&
+    \mathbb{P}(A_t=i,{{\tilde{p}}_{01}}^{*i}(t)>\frac{{{\hat{p}}_{10}^i}(t){{\hat{p}}_{01}^1}(t)}{(p_{10}^1+\epsilon_{q})},\textcolor{black}{E^{0,1}_{1,i,t}})\\& 
+     \mathbb{P}(A_t=i,{{\tilde{p}}_{10}}^{*i}(t)<\frac{{{\hat{p}}_{01}^i}(t)(p_{10}^1+\epsilon_{q})}{{{\hat{p}}_{01}^1}(t)},\textcolor{black}{E^{1,1}_{1,i,t}}).
\end{aligned}
\end{equation}
Now, we proceed to derive upper bounds on individual terms of Equation (\ref{eq:regret_H_2_Markovian}). Using Proposition \ref{lemma_H:1_Markovian},
\begin{equation*}
\begin{aligned}
&   \sum_{t=\tau_{1,i}+1}^n \mathbb{P}(A_t=i,{\tilde{p}}_{01}^{*1}(t)<p_{01}^1-\epsilon_p,S_1(t)=0,E_{1,i,t})\\& 
\le \sum_{t=1}^n \mathbb{P}({\tilde{p}}_{01}^{*1}(t)<p_{01}^1-\epsilon_p)
\le \frac{2}{{\epsilon_{p}}^2},
\end{aligned}
\end{equation*}
\begin{equation*}
\begin{aligned}
&   \sum_{t=\tau_{1,i}+1}^n \mathbb{P}(A_t=i,{\tilde{p}}_{10}^{*1}(t)>p_{10}^1+\epsilon_q,S_1(t)=1,E_{1,i,t})
\le \frac{2}{{\epsilon_{q}}^2}.  
\end{aligned}
\end{equation*}
\begin{equation*}
\begin{aligned}
&\sum_{t=\tau_{1,i}+1}^n\mathbb{P}(A_t=i,{{\tilde{p}}_{01}}^{*i}(t)>\frac{(p_{01}^1-\epsilon_{p}){{\hat{p}}_{10}^i}(t)}{{{\hat{p}}_{10}^1}(t)},\textcolor{black}{E^{0,0}_{1,i,t}})\\&
=\sum_{t=\tau_{1,i}+1}^n\mathbb{P}(A_t=i,{{\tilde{p}}_{01}}^{*i}(t)>\frac{(p_{01}^1-\epsilon_{p}){{\hat{p}}_{10}^i}(t)}{{{\hat{p}}_{10}^1}(t)},\\&D({{\hat{p}}_{01}^i}(t)||{{\tilde{p}}_{01}}^{*i}(t))
 \le \frac{\log f(t)}{T_i(t-1)},\textcolor{black}{E^{0,0}_{1,i,t}})\\&
\le \sum_{t=\tau_{1,i}+1}^n \mathbb{P}(A_t=i,{{\tilde{p}}_{01}}^{*i}(t)>\frac{(p_{01}^1-\epsilon_{p}){{\hat{p}}_{10}^i}(t)}{{{\hat{p}}_{10}^1}(t)},\\&D({{\hat{p}}_{01}^i}(t)||{{\tilde{p}}_{01}}^{*i}(t))\le \frac{\log f(t)}{T_i(t-1)},F_{1,i,t})+\sum_{t=\tau_{1,i}+1}^n \mathbb{P}(F_{1,i,t}^c)\\&
\le \sum_{t=\tau_{1,i}+1}^n \mathbb{P}(A_t=i,{{\tilde{p}}_{01}}^{*i}(t)>\frac{(p_{01}^1-\epsilon_{p})({p}_{10}^i-\epsilon_1)}{p_{10}^1+\epsilon_1},\\&D(p_{01}^i+\epsilon_1||{{\tilde{p}}_{01}}^{*i}(t))\le \frac{\log f(n)}{T_i(t-1)})+\sum_{t=\tau_{1,i}+1}^n \mathbb{P}(F_{1,i,t}^c)\\&
\le (\frac{\log f(n)}{D(p_{01}^i+\epsilon_1||\frac{(p_{01}^1-\epsilon_p)(p_{10}^i-\epsilon_1)}{p_{10}^1+\epsilon_1})}-\tau_{1,i})^++\sum_{t=\tau_{1,i}+1}^n\mathbb{P}(F_{1,i,t}^c),
\end{aligned}
\end{equation*}
where the first equality follows from the definition of ${{\tilde{p}}_{01}}^{*i}(t)$. The second inequality follows from the fact that
 $D(x||{{\tilde{p}}_{01}}^{*i}(t))$ is decreasing in $x \in [p_{01}^i-\epsilon_1,p_{01}^i+\epsilon_1]$ and
 $p_{01}^i+\epsilon_1<\frac{(p_{01}^1-\epsilon_p)(p_{10}^i-\epsilon_1)}{p_{10}^1+\epsilon_1}$. 
 The last inequality follows from Proposition \ref{lemma:2a}. Similarly,
\begin{equation*}
\begin{aligned}
&\sum_{t=\tau_{1,i}+1}^n     \mathbb{P}(A_t=i,{{\tilde{p}}_{10}}^{*i}(t)<\frac{{{\hat{p}}_{01}^i}(t){{\hat{p}}_{10}^1}(t)}{(p_{01}^1-\epsilon_{p})},\textcolor{black}{E^{1,0}_{1,i,t}})\\&
\le 
(\frac{\log f(n)}{D(p_{10}^i-\epsilon_1||\frac{(p_{10}^1+\epsilon_{1})(p_{01}^i+\epsilon_1)}{p_{01}^1-\epsilon_p})}-\tau_{1,i})^+
+\sum_{t=\tau_{1,i}+1}^n\mathbb{P}(F_{1,i,t}^c),
\end{aligned}    
\end{equation*}
\begin{equation*}
\begin{aligned}
&\sum_{t=\tau_{1,i}+1}^n     \mathbb{P}(A_t=i,{{\tilde{p}}_{01}}^{*i}(t)>\frac{{{\hat{p}}_{10}^i}(t){{\hat{p}}_{01}^1}(t)}{(p_{10}^1+\epsilon_{q})},\textcolor{black}{E^{0,1}_{1,i,t}})
\\&
\le (\frac{\log f(n)}{D(p_{01}^i+\epsilon_1||\frac{(p_{01}^1-\epsilon_{1})(p_{10}^i-\epsilon_1)}{p_{10}^1+\epsilon_q})}-\tau_{1,i})^+
+\sum_{t=\tau_{1,i}+1}^n\mathbb{P}(F_{1,i,t}^c),
\end{aligned}    
\end{equation*}
\begin{equation*}
\begin{aligned}
&\sum_{t=\tau_{1,i}+1}^n     \mathbb{P}(A_t=i,{{\tilde{p}}_{10}}^{*i}(t)<\frac{{{\hat{p}}_{01}^i}(t)(p_{10}^1+\epsilon_{q})}{{{\hat{p}}_{01}^1}(t)},\textcolor{black}{E^{1,1}_{1,i,t}})\\&\le 
(\frac{\log f(n)}{D(p_{10}^i-\epsilon_1||\frac{(p_{10}^1+\epsilon_{q})(p_{01}^i+\epsilon_1)}{p_{01}^1-\epsilon_1})}-\tau_{1,i})^++\sum_{t=\tau_{1,i}+1}^n\mathbb{P}(F_{1,i,t}^c).
\end{aligned}
\end{equation*}
Choose
$\tau_{1,i}=t_{1,i}
+\frac{\log f(n)}{D(p_{01}^i+\epsilon_1||\frac{(p_{01}^1-\epsilon_{p})(p_{10}^i-\epsilon_1)}{p_{10}^1+\epsilon_1})}+\\\frac{\log f(n)}{D(p_{10}^i-\epsilon_1||\frac{(p_{10}^1+\epsilon_{1})(p_{01}^i+\epsilon_1)}{p_{01}^1-\epsilon_p})}+
 \frac{\log f(n)}{D(p_{01}^i+\epsilon_1||\frac{(p_{01}^1-\epsilon_{1})(p_{10}^i-\epsilon_1)}{p_{10}^1+\epsilon_q})}+\\
 \frac{\log f(n)}{D(p_{10}^i-\epsilon_1||\frac{(p_{10}^1+\epsilon_{q})(p_{01}^i+\epsilon_1)}{p_{01}^1-\epsilon_1})}$.
Therefore, we obtain (using Equations (\ref{eq:e_1_tc}), (\ref{eq:f_1_t_c}) and (\ref{eq:regret_H_2_Markovian}))
\begin{equation*}
\begin{aligned}
 &   \mathbb{E}[T_i(n)]
\le
\tau_{1,i} +\frac{20}{{\epsilon_1}^2}+\frac{2}{{\epsilon}_{p}^2}+\frac{2}{{\epsilon}_{q}^2}\\&+
10\sum_{t=1}^{\infty} (t+1)^3  \exp(-2(t-1){\epsilon_1}^2(p_{01}^i+p_{10}^i)^2)\\&+
10\sum_{t=1}^{\infty} (t+1)^3  \exp(-2(t-1){\epsilon_1}^2(p_{01}^1+p_{10}^1)^2).
 \end{aligned}
 \end{equation*}
We choose $\epsilon_1=\epsilon_{p}=\epsilon_{q}=\log^{-1/4} (n)$ to complete the proof.
\subsubsection{i.i.d. Optimal and Truly Markovian Suboptimal Arms}
Similar to the first case, we prove that the expected number of times the optimal arm is in STP\_PHASE and the suboptimal arm is in SM\_PHASE, is finite.
Let $E_{2,i,t}:=\{\textcolor{black}{TV}({\hat{p}}_{01}^1(t-1)||1-{\hat{p}}_{10}^1(t-1))\le \frac{1}{{(t-1)}^{1/4}},\textcolor{black}{TV}({\hat{p}}_{01}^i(t-1)||1-{\hat{p}}_{10}^i(t-1))> \frac{1}{{(t-1)}^{1/4}}\}$.\\
Hence, for\\ $t_{2,i}=
\frac{\mathbbm{1}\{p_{01}^i+p_{10}^i<1\}}{\textcolor{black}{TV^4}(p_{01}^i+\epsilon_1||1-{p}_{10}^i-\epsilon_1)}+
\frac{\mathbbm{1}\{p_{01}^i+p_{10}^i>1\}}{\textcolor{black}{TV^4}(p_{01}^i-\epsilon_1||1-{p}_{10}^i+\epsilon_1)}$,   
we obtain (using Proposition \ref{lemma:Hellinger})
\begin{equation}\label{eq:e_2_tc}
\begin{aligned}
&  \sum_{t=t_{2,i}}^n  \mathbb{P}(E_{2,i,t}^c)\\&
\le \frac{2}{{\epsilon_1}^2}+2 \sum_{t=1}^{\infty} (t+1)^3  \exp(-2(t-1){\epsilon_1}^2(p_{01}^i+p_{10}^i)^2)\\&
+\sum_{t=1}^{\infty} 4 \exp(-\frac{2}{9}\sqrt{t})+\sum_{t=1}^{\infty} (t+1)^3  \exp(-2\sqrt{t-1}),
    \end{aligned}
\end{equation}
Assuming
$F_{2,i,t}:=\{|{{\hat{p}}_{01}^i}(t)-p_{01}^i|\le\epsilon_1,|{{\hat{p}}_{10}^i}(t)-p_{10}^i|\le\epsilon_1\}$ and using Proposition \ref{lemma:concentratation}, we get    
\begin{equation}\label{eq:f_2_tc}
\begin{aligned}
&   \sum_{t=1}^n \mathbb{P}(F_{2,i,t}^c)
\\&
\le     \frac{2}{{\epsilon_1}^2}+
2\sum_{t=1}^{\infty} (t+1)^3  \exp(-2(t-1){\epsilon_1}^2(p_{01}^i+p_{10}^i)^2).
    \end{aligned}
\end{equation}
Similar to Equation (\ref{eq:regret_H_Markovian}),
 \begin{equation}\label{eq:regret_H}
 \begin{aligned}
    & \mathbb{E}[T_i(n)]
    \le \tau_{2,i} +\sum_{t=\tau_{2,i}+1}^n\mathbb{P}(E_{2,i,t}^c)+\sum_{t=\tau_{2,i}+1}^n\mathbb{P}(A_t=i,E_{2,i,t}).
 \end{aligned}
 \end{equation}
 After each arm is chosen once, 
sub-optimal arm $i$ is chosen if at least one of these conditions is true.
 \begin{enumerate}
     \item $\tilde{\mu}^{*1}(t)<\mu_1-\epsilon_{\mu}$ ,
     \item $\mu_1-\epsilon_{\mu}<\frac{{\tilde{p}}_{01}^{*i}(t)}{{{\tilde{p}}_{01}}^{*i}(t)+{{\hat{p}}_{10}^i}(t)}$ and $S_i(t)=0$ ,
     \item $\mu_1-\epsilon_{\mu}<\frac{{{\hat{p}}_{01}^i}(t)}{{{\hat{p}}_{01}^i}(t)+{{\tilde{p}}_{10}}^{*i}(t)}$ and $S_i(t)=1$.
 \end{enumerate}
Therefore,
 \begin{equation}\label{eq:regret_H_2}
 \begin{aligned}
&  \sum_{t=\tau_{2,i}+1}^n   \mathbb{P}(A_t=i,E_{2,i,t})\\&\le \sum_{t=\tau_{2,i}+1}^n \mathbb{P}(A_t=i,\tilde{\mu}^{*1}(t)<\mu_1-\epsilon_{\mu},E_{2,i,t})\\&+
     \mathbb{P}(A_t=i,{{\tilde{p}}_{01}}^{*i}(t)>\frac{({\mu}_1-\epsilon_{\mu}) {{\hat{p}}_{10}^i}(t)}{{1-\mu_1+\epsilon_{\mu}}},S_i(t)=0,E_{2,i,t})\\&
+  \mathbb{P}(A_t=i,{{\tilde{p}}_{10}}^{*i}(t)<\frac{{{\hat{p}}_{01}^i}(t)(1-{\mu}_1+\epsilon_{\mu})}{{\mu}_1-\epsilon_{\mu}},S_i(t)=1,E_{2,i,t}).
\end{aligned}
\end{equation}
Similar to the previous case, we get (using Propositions \ref{lemma_H:1_Markovian} and \ref{lemma:2a})
\begin{equation*}
\begin{aligned}
&    \sum_{t=\tau_{2,i}+1}^n \mathbb{P}(A_t=i,\tilde{\mu}^{*1}(t)<\mu_1-\epsilon_{\mu},E_{2,i,t})
\le \frac{2}{{\epsilon_{\mu}}^2},  
\end{aligned}
\end{equation*}
\begin{equation*}
\begin{aligned}
&\sum_{t=\tau_{2,i}+1}^n\mathbb{P}(A_t=i,{{\tilde{p}}_{01}}^{*i}(t)>\frac{({\mu}_1-\epsilon_{\mu}) {{\hat{p}}_{10}^i}(t)}{{1-\mu_1+\epsilon_{\mu}}},S_i(t)=0,E_{2,i,t})\\&
\le (\frac{\log f(n)}{D(p_{01}^i+\epsilon_1||\frac{(\mu_1-\epsilon_{\mu})(p_{10}^i-\epsilon_1)}{1-\mu_1+\epsilon_{\mu}})}-\tau_{2,i})^+
+
\sum_{t=\tau+1}^n \mathbb{P}(F_{2,i,t}^c),
\end{aligned}
\end{equation*}
\begin{equation*}
\begin{aligned}
&\sum_{t=\tau_{2,i}+1}^n \mathbb{P}(A_t=i,{{\tilde{p}}_{10}}^{*i}(t)<\frac{{{\hat{p}}_{01}^i}(t)(1-{\mu}_1+\epsilon_{\mu})}{({\mu}_1-\epsilon_{\mu})},S_i(t)=1,E_{2,i,t})
\\&
\le  (\frac{\log f(n)}{D(p_{10}^i-\epsilon_1||\frac{(p_{01}^i+\epsilon_1)(1-\mu_1+\epsilon_{\mu})}{\mu_1-\epsilon_{\mu}})}-\tau_{2,i})^++\sum_{t=\tau_{2,i}+1}^n \mathbb{P}(F_{2,i,t}^c).
\end{aligned}
\end{equation*}
Choose
$\tau_{2,i}=t_{2,i}+{\frac{\log f(n)}{D(p_{01}^i+\epsilon_1||\frac{(\mu_1-\epsilon_{\mu})(p_{10}^i-\epsilon_1)}{1-\mu_1+\epsilon_{\mu}})}}+ 
\frac{\log f(n)}{D(p_{10}^i-\epsilon_1||\frac{(p_{01}^i+\epsilon_1)(1-\mu_1+\epsilon_{\mu})}{\mu_1-\epsilon_{\mu}})}.$\\
The proof follows by choosing $\epsilon_1=\epsilon_{\mu}=\log^{-1/4} (n)$ and
using Equations (\ref{eq:e_2_tc}), (\ref{eq:f_2_tc}) and (\ref{eq:regret_H_2}).
\subsubsection{Truly Markovian Optimal and i.i.d Suboptimal Arms}
Let $E_{3,i,t}:=\{\textcolor{black}{TV}({\hat{p}}_{01}^1(t-1)||1-{\hat{p}}_{10}^1(t-1))> \frac{1}{{(t-1)}^{1/4}},\textcolor{black}{TV}({\hat{p}}_{01}^i(t-1)||1-{\hat{p}}_{10}^i(t-1))\le \frac{1}{{(t-1)}^{1/4}}\}$. 
\\
Choose
$t_{3,i}= 
\frac{\mathbbm{1}\{p_{01}^1+p_{10}^1<1\}}{\textcolor{black}{TV^4}(p_{01}^1+\epsilon_1||1-{p}_{10}^1-\epsilon_1)}+
\frac{\mathbbm{1}\{p_{01}^1+p_{10}^1>1\}}{\textcolor{black}{TV^4}(p_{01}^1-\epsilon_1||1-{p}_{10}^1+\epsilon_1)}.$ 
Similar to Equation (\ref{eq:regret_H_Markovian}),
 \begin{equation*}
 \begin{aligned}
    & \mathbb{E}[T_i(n)]
    \le \tau_{3,i} +\sum_{t=\tau_{3,i}+1}^n\mathbb{P}(E_{3,i,t}^c)+\sum_{t=\tau_{3,i}+1}^n\mathbb{P}(A_t=i,E_{3,i,t}).
 \end{aligned}
 \end{equation*}
 After each arm is chosen once, a sub-optimal arm is chosen 
 if at least one of following conditions is true.
 \begin{enumerate}
     \item ${\tilde{p}}_{01}^{*1}(t)<p_{01}^1-\epsilon_p$ and $S_1(t)=0$,
    \item ${\tilde{p}}_{10}^{*1}(t)>p_{10}^1+\epsilon_q$ and  $S_1(t)=1$,
     \item $\tilde{\mu}^{*i}(t)>\frac{p_{01}^1-\epsilon_p}{p_{01}^1-\epsilon_p+{{\hat{p}}_{10}^1}(t)}$ and $S_1(t)=0$,
     \item $\tilde{\mu}^{*i}(t)>\frac{{{\hat{p}}_{01}^1}(t)}{{{\hat{p}}_{01}^1}(t)+p_{10}^1+\epsilon_q}$ and $S_1(t)=1$.
 \end{enumerate}
\textcolor{black}{We follow identical set of procedures and use Proposition \ref{lemma_H:1_Markovian}.}
We Choose
$\tau_{3,i}=t_{3,i}+\frac{\log f(n)}{D(\mu_i+\epsilon_1||\frac{p_{01}^1-\epsilon_{p}}{p_{01}^1-\epsilon_{p}+p_{10}^1+\epsilon_1})}+\frac{\log f(n)}{D(\mu_i+\epsilon_1||\frac{p_{01}^1-\epsilon_{1}}{p_{01}^1-\epsilon_{1}+p_{10}^1+\epsilon_{q}})}.$\\
The proof follows by 
choosing $\epsilon_1=\epsilon_{p}=\epsilon_{q}=\log^{-1/4} (n)$ .
\subsubsection{i.i.d. Arms:}
Let $E_{4,i,t}:=\{\textcolor{black}{TV}({\hat{p}}_{01}^1(t-1)||1-{\hat{p}}_{10}^1(t-1))\le  \frac{1}{{(t-1)}^{1/4}},\textcolor{black}{TV}({\hat{p}}_{01}^i(t-1)||1-{\hat{p}}_{10}^i(t-1))\le \frac{1}{{(t-1)}^{1/4}}\}$.
 After each arm is chosen once, 
 a sub-optimal arm is chosen 
 if either ${\tilde{\mu}}^{*1}(t)<\mu_1-\epsilon_{\mu}$ or ${\tilde{\mu}}^{*i}(t)>\mu_1-\epsilon_{\mu}$. 
Using Propositions \ref{lemma:Hellinger} and
\ref{lemma:2a}, we get
\begin{equation*}
\begin{aligned}
 &   \mathbb{E}[T_i(n)]
 \le \tau_{4,i} +\sum_{t=\tau_{4,i}+1}^n\mathbb{P}(E_{4,i,t}^c)+ 
 \frac{2}{{\epsilon}_{\mu}^2}+\frac{1}{2{\epsilon}_{1}^2}\\&+(\frac{\log f(n)}{D(\mu_i+\epsilon_1||\mu_1-\epsilon_1)}-\tau_{4,i})^+.
\end{aligned}
\end{equation*}
We choose 
$\tau_{4,i}=\frac{\log f(n)}{D(\mu_i+\epsilon_1||\mu_1-\epsilon_1)}$ and $\epsilon_1=\epsilon_{\mu}=\log^{-1/4} (n)$ to complete the proof. \qed
\section{}\label{app:234}
\textbf{Proof of Proposition \ref{lemma:mixing}:}
We compute the probability that the Markov chain observes the sequence $m_1,m_2,\ldots,m_t$. 
Let $\hat{\pi}(i,t)=\frac{N_i(t)}{t-1}$.
\begin{equation}\label{eq:Type_11}
\begin{aligned}
&\mathbb{P}((M_1,M_2,\ldots, M_t)=(m_1,m_2,\ldots,m_t))\\& 
=\mathbb{P}(M_2=m_2|M_1=m_1)\ldots \mathbb{P}(M_t=m_t|M_{t-1}=m_{t-1})\\&
=\exp((t-1)\sum_{i \in \mathcal{S}} \hat{\pi}(i,t) \sum_{j \in \mathcal{S}} \hat{P}_{ij}(t)\log P_{ij}).
\end{aligned}
\end{equation}
Let $T=\{\hat{\pi}(i,t),\hat{P}_{ij}(t)\}$ be the type of the sequence. We want to find the number of sequences that have type $T$. We consider a Markov chain where the transition probability from state $i$ to state $j$ is $\hat{P}_{ij}(t)$. Let $N(T)$ be the number of sequences whose type is $T$. We have, 
\begin{equation}\label{eq:Type_22}
 1\ge N(T)\exp((t-1)\sum_{i\in \mathcal{S}} \hat{\pi}(i,t) \sum_{j \in \mathcal{S}} \hat{P}_{ij}(t)\log \hat{P}_{ij}(t)). \end{equation}
Let $\mathcal{H}$ be the set of types which satisfy  $|\frac{N_0(t)}{t-1}-\pi_0|>\epsilon_1$. Clearly, types in $\mathcal{H}$ also satisfy  $|\frac{N_1(t)}{t-1}-\pi_1|>\epsilon_1$. 
Hence, 
\begin{equation*}
\begin{aligned}
&\mathbb{P} (|\frac{N_0(t)}{t-1}-\pi_0|>\epsilon_1)=\sum_{T \in \mathcal{H}} N(T)P(T)\\&
=\sum_{T \in \mathcal{H}} \exp(-(t-1)(\hat{\pi}(0,t) D(\hat{P}_{01}(t)||P_{01})\\&+\hat{\pi}(1,t) D(\hat{P}_{10}(t)||P_{10})))\\&
\le(t+1)^3  \exp(-(t-1)\min_{T \in \mathcal{H}} (\hat{\pi}(0,t) D(\hat{P}_{01}(t)||P_{01})\\&+\hat{\pi}(1,t) D(\hat{P}_{10}(t)||P_{10}))),
\end{aligned}
\end{equation*}
where the equality follows from  (\ref{eq:Type_11}) and (\ref{eq:Type_22}). The inequality follows since each of $\hat{P}_{01}(t)$,$\hat{P}_{10}(t),\hat{\pi}(0,t)$  can take $(t+1)$ possible values and hence, $|\mathcal{H}|\le (t+1)^3$. Therefore the
problem reduces to the following problem.
\begin{equation*}
    \begin{aligned}
&        \min\limits_{x,y} \frac{y}{x+y} D(x||P_{01})+\frac{x}{x+y} D(y||P_{10}),\\&
        \text{subject to }  
|\frac{y}{x+y}-\pi_0|>\epsilon_1,|\frac{x}{x+y}-\pi_1|>\epsilon_1,\\&
 x\ge 0,
 y\ge 0,
 x\le 1,
 y\le 1.
    \end{aligned}
\end{equation*}
Let $(x^*,y^*)$ be a solution to the above problem. Using Pinsker's inequality,
$D(x^*||P_{01})\ge 2(x^*-P_{01})^2$ and 
$D(y^*||P_{10})\ge 2(y^*-P_{10})^2$.
We know that $(\frac{x^*}{x^*+y^*}-\frac{P_{01}}{P_{01}+P_{10}})^2>{\epsilon_1}^2$ and $(\frac{y^*}{x^*+y^*}-\frac{P_{10}}{P_{01}+P_{10}})^2>{\epsilon_1}^2$.
Hence,
$\frac{y^*D(x^*||P_{01})}{x^*+y^*}+\frac{x^*D(y^*||P_{10})}{x^*+y^*}
> 2{\epsilon_1}^2(P_{01}+P_{10})^2$.
The rest of the proof follows immediately.\qed\\
\textbf{Proof of Proposition \ref{lemma:concentratation}:}
We take
$C_t:=\{|\frac{N_0(t)}{t-1}-\pi_0|>\epsilon_1\}$.
\begin{equation*}
\begin{aligned}
& \sum_{t=1}^{n}\mathbb{P}(D_{0,t})=\sum_{t=1}^{n}\mathbb{P}(D_{0,t},C_t^c)+\mathbb{P}(D_{0,t},C_t)\\&
\le \sum_{t=1}^n \mathbb{P}(|\hat{P}_{01,N_0(t)}-P_{01}|>\epsilon_1,|\frac{N_0(t)}{t-1}-\pi|\le \epsilon_1)+\mathbb{P}(C_t)\\&
\le \sum_{s_0=1}^{\infty}\mathbb{P}(|\hat{P}_{01,s_0}-P_{01}|>\epsilon_1)
+\sum_{t=1}^{n}\mathbb{P}(C_t)\\&
\le \sum_{s_0=1}^{\infty}2\exp(-2s_0{\epsilon_1}^2)+\sum_{t=1}^{n}\mathbb{P}(C_t)\\&
\le \frac{1}{{\epsilon_1}^2}+
\sum_{t=1}^{\infty} (t+1)^3  \exp(-2(t-1){\epsilon_1}^2(P_{01}+P_{10})^2),
\end{aligned}    
\end{equation*}
where the third and last inequalities follow from Propositions \ref{lemma:pinsker} and 
\ref{lemma:mixing}, respectively.
Proof for $\sum \limits_{t=1}^n\mathbb{P}(D_{1,t})$ follows in a similar manner.\qed\\
\textbf{Proof of Proposition \ref{lemma:Hellinger}:}
\textit{Case I-$P_{01}+P_{10}=1$:} Using triangle inequality,
 $ \textcolor{black}{TV}(\hat{P}_{01}(t)||1-\hat{P}_{10}(t))
\le |\hat{P}_{01}(t)-P_{01}|+|\hat{P}_{10}(t)-P_{10}|+|P_{01}+P_{10}-1|$,      
We take $C_t:=\{|\frac{N_0(t)}{t-1}-\pi_0|>\epsilon\}$ and $\epsilon=\frac{1}{{(t-1)}^{1/4}}$.
Hence, using Propositions \ref{lemma:pinsker} and  \ref{lemma:mixing},
\begin{equation*}
\begin{aligned}
  &  \sum_{t=1}^n \mathbb{P}(\textcolor{black}{TV}(\hat{P}_{01}(t)||1-\hat{P}_{10}(t))\ge \frac{1}{t^{1/4}})\\&
 \le \sum_{t=1}^n \mathbb{P}( |\hat{P}_{01}(t)-P_{01}|+|\hat{P}_{10}(t)-P_{10}|\\&+|P_{01}+P_{10}-1|\ge \frac{1}{t^{1/4}},C_t^c)+
  \mathbb{P}(C_t)\\&
    \le  \sum_{s_0=1}^\infty  \mathbb{P}(3|\hat{P}_{01,s_0}-P_{01}|\ge \frac{1}{{s_0}^{1/4}})\\& +\sum_{s_1=1}^\infty \mathbb{P}(3|\hat{P}_{10,s_1}-P_{10}|\ge \frac{1}{{s_1}^{1/4}})
    +\sum_{t=1}^n\mathbb{P}(C_t)\\&
    \le \sum_{t=1}^{\infty} 4 \exp(-\frac{2}{9}\sqrt{t})+\sum_{t=1}^{\infty} (t+1)^3 \exp(-2\sqrt{t-1}).
\end{aligned}
\end{equation*}
\textit{Case II-$P_{01}+P_{10}\neq 1$:}
We take
$D_t:=\{|\hat{P}_{10}(t)-P_{10}|<\epsilon_1,|\hat{P}_{01}(t)-P_{01}|< \epsilon_1\}$ and choose $0<\epsilon_1<\frac{|P_{01}+P_{10}-1|}{2}$.
Hence,
\begin{equation*}
\begin{aligned}
  &  \sum_{t=\tau}^n \mathbb{P}(\textcolor{black}{TV}(\hat{P}_{01}(t)||1-\hat{P}_{10}(t))< \frac{1}{t^{1/4}})\\&
  \le \sum_{t=\tau}^n \mathbb{P}(\textcolor{black}{TV}(\hat{P}_{01}(t)||1-\hat{P}_{10}(t))< \frac{1}{t^{1/4}},D_t)+
  \mathbb{P}(D_t^c).
  \end{aligned}
  \end{equation*}
Now, for $P_{01}+P_{10}<1$,
if $\tau>\frac{1}{\textcolor{black}{TV}^4(P_{01}+\epsilon_1||1-{P}_{10}-\epsilon_1)}$,
\begin{equation*}
\begin{aligned}
  &  \sum_{t=\tau}^n \mathbb{P}(\textcolor{black}{TV}(\hat{P}_{01}(t)||1-\hat{P}_{10}(t))< \frac{1}{t^{1/4}})\\&
  \le \sum_{t=\tau}^n \mathbb{P}(\textcolor{black}{TV}(P_{01}+\epsilon_1||1-{P}_{10}-\epsilon_1)< \frac{1}{t^{1/4}})+
  \mathbb{P}(D_t^c)\\&
\le  \sum_{t=1}^n\mathbb{P}(|\hat{P}_{10}(t)-P_{10}|>\epsilon_1)+\sum_{t=1}^n\mathbb{P}(|\hat{P}_{01}(t)-P_{01}|>\epsilon_1).
  \end{aligned}
  \end{equation*}
  Similar result holds for 
$P_{01}+P_{10}>1$.
  The rest of the proof follows from Proposition \ref{lemma:concentratation}.\qed
\\\textbf{Proof of Proposition \ref{lemma_H:1_Markovian}:}
\begin{equation*}
 \begin{aligned}
 & \mathbb{P}(\tilde{p}_s^*<p-\epsilon_{p})
\le      \mathbb{P}\big (D(\hat{p}_s||p-\epsilon_{p})> \frac{\log f(t)}{s},\hat{p}_s< p-\epsilon_{p} \big)\\&  
        \le   \mathbb{P}\big (D(\hat{p}_s||p)> \frac{\log f(t)}{s}+2{\epsilon_{p}}^2,\hat{p}_s< p\big) \\&
        \le \sum_{s=1}^n \mathbb{P}\big (D(\hat{p}_s||p)> \frac{\log f(t)}{s}+2{\epsilon_{p}}^2,\hat{p}_s< p\big)\\&
        \le  \sum_{s=1}^n \exp{\Big(-s\Big(2{\epsilon_{p}}^2+ \frac{\log f(t)}{s}\Big)\Big)}\\&
        \le 
        \sum_{s=1}^{\infty}\frac{1}{f(t)}  \exp(-2s{\epsilon_{p}}^2) \le \frac{1}{2{\epsilon_{p}}^2f(t)}, 
 \end{aligned}
 \end{equation*}
 where the second and fourth inequalities follow from Proposition \ref{lemma:pinsker} and Corollary
\ref{corollary:1}, respectively.
The last inequality uses the fact that $\sum \limits_{s=1}^{\infty} \exp(-sx)= \frac{1}{e^x-1}\le \frac{1}{x}$.
 Hence, 
  \begin{equation*}
 \begin{aligned}
 &\sum_{s=1}^{n}\mathbb{P}(\tilde{p}_s^{*}<p-\epsilon_{p})
\le \sum_{t=1}^{\infty}\frac{1}{2{\epsilon_{p}}^2f(t)}\le \int_{0}^{\infty} \frac{dt}{2{\epsilon_{p}}^2f(t)}\le \frac{2}{{\epsilon_{p}}^2}.
\end{aligned}
\end{equation*}
The last inequality follows from the fact that $1+x \log^2(x)\ge x \log^2(x)$ and $\int_e^{\infty} \frac{dx}{x \log^2(x)}=1$.\\
 \begin{equation*}
 \begin{aligned}
 & \mathbb{P}(\tilde{q}_s^*>q+\epsilon_{q})
\le      \mathbb{P}\big (D(\hat{q}_s||q+\epsilon_{q})> \frac{\log f(t)}{s},\hat{q}_s> q+\epsilon_{q} \big)\\&  
        \le   \mathbb{P}\big (D(\hat{q}_s||q)> \frac{\log f(t)}{s}+2{\epsilon_{q}}^2,\hat{q}_s> q\big) \\&
        \le \sum_{s=1}^n \mathbb{P}\big (D(\hat{q}_s||q)> \frac{\log f(t)}{s}+2{\epsilon_{q}}^2,\hat{q}_s> q\big)\\&
        \le  \sum_{s=1}^n \exp{\Big(-s\Big(2{\epsilon_{q}}^2+ \frac{\log f(t)}{s}\Big)\Big)}
        \le \frac{1}{2{\epsilon_{q}}^2f(t)}.
 \end{aligned}
 \end{equation*}
 \qed
\\ \textbf{Proof of Proposition \ref{lemma:2a}:} Since $D(p_2+\epsilon_1||\frac{(p_1-\epsilon_p)(q_2-\epsilon_1)}{q_1+\epsilon_1}+Z(s))\le \frac{a}{s}$ implies $D(p_2+\epsilon_1||\frac{(p_1-\epsilon_p)(q_2-\epsilon_1)}{q_1+\epsilon_1})\le \frac{a}{s}$,
\begin{equation*}
 \begin{aligned}
& \kappa_1
\le \sum_{s=c}^n \mathbb{P}(D(p_2+\epsilon_1||\frac{(p_1-\epsilon_p)(q_2-\epsilon_1)}{q_1+\epsilon_1})\le \frac{a}{s})\\&
= ( \frac{a}{D(p_2+\epsilon_1||\frac{(p_1-\epsilon_p)(q_2-\epsilon_1)}{q_1+\epsilon_1})}-(c-1))^+.
\end{aligned}
\end{equation*}
Similarly,
   $\kappa_2
    \le 
 (\frac{a}{D(q_2-\epsilon_1||\frac{(q_1+\epsilon_1)(p_2+\epsilon_1)}{p_1-\epsilon_p})}-(c-1))^+$.\\
Using Proposition \ref{lemma:pinsker},
\begin{equation*}
\begin{aligned}
  & \kappa_3
  \le \sum\limits_{s=1}^{\infty} \mathbb{P}(\hat{\mu}_s\ge \mu+\epsilon_1)\\&
  +\sum\limits_{s=c}^{n} \mathbb{P}(D(\hat{\mu}_s||\frac{p-\epsilon_p}{p-\epsilon_p+q+\epsilon_1})\le \frac{a}{s},\hat{\mu}_s< \mu+\epsilon_1)\\&
  \le \frac{1}{2{\epsilon_1}^2}+\sum_{s=c}^n \mathbb{P}(D(\mu+\epsilon_1||\frac{p-\epsilon_p}{p-\epsilon_p+q+\epsilon_1})\le \frac{a}{s})\\&
  = \frac{1}{2{\epsilon_1}^2}+(\frac{a}{D(\mu+\epsilon_1||\frac{p-\epsilon_{p}}{p-\epsilon_{p}+q+\epsilon_1})}-(c-1))^+.
\end{aligned}
\end{equation*}
\qed

\textbf{Proof of Theorem \ref{theo:eigen}:}
We know that
$\sigma_i=p_{01}^i+p_{10}^i$. 
The upper bound on the regret of UCB-SM \cite{tekin2010online} for arms modeled as two-state Markov chains ($\textcolor{black}{r(s)=s}$) is
$\sum_{i\neq 1} \frac{4L}{(\mu_1-\mu_i)},$
with $L= \frac{360}{\min \limits_i \sigma_i} $.\\
\textit{1) Truly Markovian arms:}
Using Proposition \ref{lemma:pinsker}, 
\begin{equation}\label{eq:compare_regret}
\begin{aligned}
&\sum_{i\neq 1}\Delta_i \big[\frac{\mathbbm{1}\{p_{01}^1 p_{10}^i<p_{10}^1\}}{D(p_{01}^i||\frac{p_{01}^1 p_{10}^i}{p_{10}^1})}+\frac{1}{D(p_{10}^i||\frac{p_{10}^1 p_{01}^i}{p_{01}^1})}\big]\\& \le  \sum_{i\neq 1} \frac{\Delta_i({p_{10}^1}^2+{p_{01}^1}^2)}{{2(p_{01}^ip_{10}^1-p_{01}^1p_{10}^i)}^2}
\le \sum_{i\neq 1} \frac{1}{2(\mu_1-\mu_i)\min\limits_i{\sigma_i}^2},
\end{aligned}
\end{equation}
since $\Delta_i=\frac{p_{01}^ip_{10}^1-p_{01}^1p_{10}^i}{(p_{01}^i+p_{10}^i)(p_{01}^1+p_{10}^1)}$. The 
result holds true if
$\frac{1}{\min \limits_i{(p_{01}^i+p_{10}^i)}^2}\le \frac{360*4}{\min \limits_i (p_{01}^i+p_{10}^i)}$ or $\min \limits_i \sigma_i \ge \frac{1}{1440}$.\\
\textit{2) i.i.d. optimal \& truly Markovian suboptimal:}
Proof follows directly from
Equation (\ref{eq:compare_regret}), $\mu_1=p_{01}^1$ and $1-\mu_1=p_{10}^1$.\\
\textit{3) Truly Markovian optimal \& i.i.d. suboptimal:} Using Proposition \ref{lemma:pinsker},
\begin{equation*}
\begin{aligned}
&\sum_{i\neq 1} 
\frac{2\Delta_i}{D(\mu_i||\frac{p_{01}^1}{p_{01}^1+p_{10}^1})}
=\sum_{i\neq 1} 
\frac{2\Delta_i}{D(\mu_i||\mu_1)}
\le \sum_{i\neq 1}\frac{1}{(\mu_1-\mu_i)}.
\end{aligned}   
\end{equation*}
\textit{4) i.i.d. arms:}
Since the upper bound on the regret of \textcolor{black}{TV}-KL-UCB matches
the lower bound,  
the result follows immediately.\qed
\section{}\label{app:klucbmc}
\textbf{KL-UCB-MC Algorithm:}
\begin{algorithm}
\caption{KL-UCB-MC}\label{algo:2}
\label{NCalgorithm}
\begin{algorithmic}[1]
\small
\State Input $K$ (number of arms).
\State Choose each arm once.
 \If { (state of arm $i=0$)}
 \State Update $U_i$ using Equation (\ref{eq:state0}).
\Else
\State Update $U_i$ using Equation (\ref{eq:state1}).
\EndIf
\State Choose $A_t=\arg \max \limits_i U_i$.
\end{algorithmic}
\end{algorithm}
KL-UCB-MC is a variation of standard KL-UCB for i.i.d. rewards \cite{garivier2011kl,cappe2013kullback} where one obtains a confidence bound for the estimate of $p_{01}^i$ (estimate of $p_{10}^i$) and uses the estimate of $p_{10}^i$ (estimate of $p_{01}^i$) in state $0$ (state $1$) of arm $i$ using KL-UCB. This is represented by Equations  (\ref{eq:state0}) and
(\ref{eq:state1}),
respectively. The main difference between KL-UCB-MC and \textcolor{black}{TV}-KL-UCB is that KL-UCB-MC is always in the STP\_PHASE of \textcolor{black}{TV}-KL-UCB, irrespective of the arm being truly Markovian or i.i.d. Therefore, the asymptotic upper bound on the regret of KL-UCB-MC is same as that of \textcolor{black}{TV}-KL-UCB for truly Markovian arms, irrespective of the arms being truly Markovian or i.i.d. 
The resulting asymptotic upper bound on the regret is smaller than that of \cite{tekin2010online} (See Theorem \ref{theo:eigen}). However, KL-UCB-MC results in large constants in the regret for i.i.d. rewards (given by Theorem \ref{theo:regret_2}(a)). 
Asymptotic performances of KL-UCB-MC and \textcolor{black}{TV}-KL-UCB are exactly same for arms with truly Markovian rewards.
 We know that the asymptotic upper bound on the regret of KL-UCB-MC is less than that of \cite{tekin2010online} if $\min \limits_i (p_{01}^i+p_{10}^i)\ge \frac{1}{1440}$ (See Theorem \ref{theo:eigen}). Numerical results  reveal (See Figure \ref{fig:asymptotic_1_1}) that even when this condition is not met, the asymptotic regret upper bound is less than that of \cite{tekin2010online}. \\
 \textbf{\textcolor{black}{Choice of Hellinger Distance}:}
 \label{app:rationale-Hellinger}
 \textcolor{black}{Similar to total variation distance, Hellinger distance ($H(\cdot||\cdot)$)} can be chosen 
over KL distance which is a natural choice for representing the similarity between two probability distributions in the bandit literature. \textcolor{black}{Similar to} Proposition \ref{lemma:Hellinger}, we \textcolor{black}{can} prove that 
 if $B_t:=\{H^2(\hat{P}_{01}(t)||1-\hat{P}_{10}(t))<\frac{1}{t^{1/4}}\}$, the events in the sequence $\{{B_t}\}_{t=0}^{\infty}$ occur infinitely often if $P_{01}+P_{10}=1$. This is same as proving 
 $\sum \limits_{t=1}^\infty \mathbb{P}(H^2(\hat{P}_{01}(t)||1-\hat{P}_{10}(t))\ge \frac{1}{t^{1/4}})$ is finite if $P_{01}+P_{10}=1$.  To show this, we \textcolor{black}{need to} utilize the relationship
 \begin{equation}\label{eq:app_hellinger}
 \begin{aligned}
&     H^2(\hat{P}_{01}(t)||1-\hat{P}_{10}(t))\le |\hat{P}_{01}(t)+\hat{P}_{10}(t)-1|\\& \le |\hat{P}_{01}(t)-P_{01}|+|\hat{P}_{10}(t)-P_{10}|+|P_{01}+P_{10}-1|.
     \end{aligned}
 \end{equation}
Now, if we replace $H(\cdot||\cdot)$ by $D(\cdot||\cdot)$, we need to prove that if 
$G_t:=\{D^2(\hat{P}_{01}(t)||1-\hat{P}_{10}(t))<\frac{1}{t^{1/4}}\}$, the events in the sequence $\{{G_t}\}_{t=0}^{\infty}$ occur infinitely often iff $P_{01}+P_{10}=1$.
To prove this, we need to find an appropriate upper bound on KL distance, similar to Equation (\ref{eq:app_hellinger}).
Pinsker's inequality which is a well-known bound on the KL distance, provides a lower bound and hence, cannot be used for the proof.
In \cite{dragomir2000some}, the authors propose the following upper bound 
\begin{equation*}
\begin{aligned}
&    D(\hat{P}_{01}(t)||1-\hat{P}_{10}(t))\\&\le \frac{1-2(\hat{P}_{01}(t)+\hat{P}_{10}(t))+{(\hat{P}_{01}(t)+\hat{P}_{10}(t))}^2}{\hat{P}_{10}(t)(1-\hat{P}_{10}(t))}.
\end{aligned}
\end{equation*}
Hence, unlike Equation (\ref{eq:app_hellinger}), in this case, $\hat{P}_{01}(t)$ and $\hat{P}_{10}(t)$ cannot be separated. Hence, we cannot apply Chernoff's bound to show the finiteness of $\sum \limits_{t=1}^\infty \mathbb{P}(D(\hat{P}_{01}(t)||1-\hat{P}_{10}(t))\ge \frac{1}{t^{1/4}})$. \textcolor{black}{However, results in \cite[Theorem~I.2]{agrawal2020finite} can be utilized to design similar online test using KL distance 
to detect whether an arm is truly Markovian or i.i.d.}\\
\textbf{Extension to Non-zero Reward in State $0$:}
\textcolor{black}{Although in our model, there is no reward in state 0, the proposed model and algorithm can be extended to take into account non-zero reward in state 0 easily. Recall that the mean reward and stationary distribution associated with arm $i$ are $\mu_i$ and $\pi_i$, respectively. Also, the reward obtained by playing an arm in state $s$ is $r(s)$. Hence, $\mu_i=r(0)\pi_i(0)+r(1)\pi_i(1)=r(0)\frac{p^i_{10}}{p^i_{01}+p^i_{10}}+r(1)\frac{p^i_{01}}{p^i_{01}+p^i_{10}}$.
Now, in STP\_PHASE, if the current state of arm $i$ is 0, then we compute $U_i$ in the following way.\\ 
$U_i=\sup \{r(0)\frac{\hat{p}^i_{10}}{\tilde{p}+\hat{p}^i_{10}}+r(1)\frac{\tilde{p}}{\tilde{p}+\hat{p}^i_{10}}:D(\hat{p}^i_{01}(t-1)||\tilde{p})\le \frac{\log f(t)}{T_i(t-1)}\}$.
 Similarly, in STP\_PHASE, if the current state of arm $i$ is 1, then
  $U_i=\sup \{r(0)\frac{\tilde{p}}{\tilde{p}+\hat{p}^i_{01}}+r(1)\frac{\hat{p}^i_{01}}{\tilde{p}+\hat{p}^i_{01}}:D(\hat{p}^i_{10}(t-1)||\tilde{p})\le \frac{\log f(t)}{T_i(t-1)}\}$.
 The rest of the algorithm remains unmodified.}
\section*{Acknowledgment}
This work  was  supported  by  the following grants: Navy N00014-19-1-2566, ARO W911NF-19-1-0379, NSF CMMI-1826320, ARO W911NF-17-1-0359, NSF Grants CNS 2106801 and CCF1934986 and Start-up Grant at
IIT Guwahati. The work of A. Roy was partly done while he
was with Coordinated Science Laboratory, University of Illinois at
Urbana-Champaign, Champaign, USA.
\bibliography{bandit}

\begin{thebibliography}{10}
\providecommand{\url}[1]{#1}
\csname url@samestyle\endcsname
\providecommand{\newblock}{\relax}
\providecommand{\bibinfo}[2]{#2}
\providecommand{\BIBentrySTDinterwordspacing}{\spaceskip=0pt\relax}
\providecommand{\BIBentryALTinterwordstretchfactor}{4}
\providecommand{\BIBentryALTinterwordspacing}{\spaceskip=\fontdimen2\font plus
\BIBentryALTinterwordstretchfactor\fontdimen3\font minus
  \fontdimen4\font\relax}
\providecommand{\BIBforeignlanguage}[2]{{%
\expandafter\ifx\csname l@#1\endcsname\relax
\typeout{** WARNING: IEEEtran.bst: No hyphenation pattern has been}%
\typeout{** loaded for the language `#1'. Using the pattern for}%
\typeout{** the default language instead.}%
\else
\language=\csname l@#1\endcsname
\fi
#2}}
\providecommand{\BIBdecl}{\relax}
\BIBdecl

\bibitem{robbins1952some}
H.~Robbins, ``Some aspects of the sequential design of experiments,''
  \emph{Bulletin of the American Mathematical Society}, vol.~58, no.~5, pp.
  527--535, 1952.

\bibitem{lai1985asymptotically}
T.~L. Lai and H.~Robbins, ``Asymptotically efficient adaptive allocation
  rules,'' \emph{Advances in applied mathematics}, vol.~6, no.~1, pp. 4--22,
  1985.

\bibitem{auer2002nonstochastic}
P.~Auer, N.~Cesa-Bianchi, Y.~Freund, and R.~E. Schapire, ``The nonstochastic
  multiarmed bandit problem,'' \emph{SIAM journal on computing}, vol.~32,
  no.~1, pp. 48--77, 2002.

\bibitem{thompson1933likelihood}
W.~R. Thompson, ``On the likelihood that one unknown probability exceeds
  another in view of the evidence of two samples,'' \emph{Biometrika}, vol.~25,
  no. 3/4, pp. 285--294, 1933.

\bibitem{anantharam1987asymptotically-1}
V.~Anantharam, P.~Varaiya, and J.~Walrand, ``Asymptotically efficient
  allocation rules for the multiarmed bandit problem with multiple plays-part
  i: {IID} rewards,'' \emph{IEEE Transactions on Automatic Control}, vol.~32,
  no.~11, pp. 968--976, 1987.

\bibitem{agrawal1995sample}
R.~Agrawal, ``Sample mean based index policies by $o(\log n)$ regret for the
  multi-armed bandit problem,'' \emph{Advances in Applied Probability},
  vol.~27, no.~4, pp. 1054--1078, 1995.

\bibitem{auer2002finite}
P.~Auer, N.~Cesa-Bianchi, and P.~Fischer, ``Finite-time analysis of the
  multiarmed bandit problem,'' \emph{Machine learning}, vol.~47, no. 2-3, pp.
  235--256, 2002.

\bibitem{garivier2011kl}
A.~Garivier and O.~Capp{\'e}, ``The {KL-UCB} algorithm for bounded stochastic
  bandits and beyond,'' in \emph{conference on learning theory}, 2011, pp.
  359--376.

\bibitem{cappe2013kullback}
O.~Capp{\'e}, A.~Garivier, O.-A. Maillard, R.~Munos, and G.~Stoltz,
  ``{Kullback--Leibler} upper confidence bounds for optimal sequential
  allocation,'' \emph{The Annals of Statistics}, vol.~41, no.~3, pp.
  1516--1541, 2013.

\bibitem{tekin2010online}
C.~Tekin and M.~Liu, ``Online algorithms for the multi-armed bandit problem
  with {M}arkovian rewards,'' in \emph{IEEE Annual Allerton Conference on
  Communication, Control, and Computinga}, 2010, pp. 1675--1682.

\bibitem{anantharam1987asymptotically-2}
V.~Anantharam, P.~Varaiya, and J.~Walrand, ``Asymptotically efficient
  allocation rules for the multiarmed bandit problem with multiple plays-part
  ii: {M}arkovian rewards,'' \emph{IEEE Transactions on Automatic Control},
  vol.~32, no.~11, pp. 977--982, 1987.

\bibitem{tekin2012online}
C.~Tekin and M.~Liu, ``Online learning of rested and restless bandits,''
  \emph{IEEE Transactions on Information Theory}, vol.~58, no.~8, pp.
  5588--5611, 2012.

\bibitem{moulos2020finite}
V.~Moulos, ``Finite-time analysis of round-robin kullback-leibler upper
  confidence bounds for optimal adaptive allocation with multiple plays and
  markovian rewards,'' \emph{Advances in Neural Information Processing
  Systems}, vol.~33, pp. 7863--7874, 2020.

\bibitem{liu2012learning}
H.~Liu, K.~Liu, and Q.~Zhao, ``Learning in a changing world: Restless
  multiarmed bandit with unknown dynamics,'' \emph{IEEE Transactions on
  Information Theory}, vol.~59, no.~3, pp. 1902--1916, 2012.

\bibitem{tekin2011online}
C.~Tekin and M.~Liu, ``Online learning in opportunistic spectrum access: A
  restless bandit approach,'' in \emph{IEEE INFOCOM}, 2011, pp. 2462--2470.

\bibitem{cortes2017}
C.~Cortes, G.~DeSalvo, V.~Kuznetsov, M.~Mohri, and S.~Yang, ``Discrepancy-based
  algorithms for non-stationary rested bandits,'' \emph{arXiv preprint
  arXiv:1710.10657}, 2017.

\bibitem{graepel2010web}
T.~Graepel, J.~Q. Candela, T.~Borchert, and R.~Herbrich, ``Web-scale bayesian
  click-through rate prediction for sponsored search advertising in microsoft's
  bing search engine.''\hskip 1em plus 0.5em minus 0.4em\relax Omnipress, 2010.

\bibitem{herbrich2006trueskill}
R.~Herbrich, T.~Minka, and T.~Graepel, ``Trueskill™: A bayesian skill rating
  system,'' in \emph{Proceedings of the 19th international conference on neural
  information processing systems}, 2006, pp. 569--576.

\bibitem{lattimore2018bandit}
T.~Lattimore and C.~Szepesv{\'a}ri, ``Bandit algorithms,'' \emph{preprint},
  2018.

\bibitem{avadhanula2021stochastic}
V.~Avadhanula, R.~Colini~Baldeschi, S.~Leonardi, K.~A. Sankararaman, and
  O.~Schrijvers, ``Stochastic bandits for multi-platform budget optimization in
  online advertising,'' in \emph{WWW}, 2021, pp. 2805--2817.

\bibitem{zheng2016sequential}
R.~Zheng and C.~Hua, \emph{Sequential Learning and Decision-Making in Wireless
  Resource Management}.\hskip 1em plus 0.5em minus 0.4em\relax Springer, 2016.

\bibitem{gittins1974dynamic}
J.~Gittins, ``A dynamic allocation index for the sequential design of
  experiments,'' \emph{Progress in statistics}, pp. 241--266, 1974.

\bibitem{whittle1988restless}
P.~Whittle, ``Restless bandits: Activity allocation in a changing world,''
  \emph{Journal of applied probability}, vol.~25, no.~A, pp. 287--298, 1988.

\bibitem{lazaric2014online}
A.~Lazaric, E.~Brunskill \emph{et~al.}, ``Online stochastic optimization under
  correlated bandit feedback,'' in \emph{International Conference on Machine
  Learning}.\hskip 1em plus 0.5em minus 0.4em\relax PMLR, 2014, pp. 1557--1565.

\bibitem{foster1999regret}
D.~P. Foster and R.~Vohra, ``Regret in the on-line decision problem,''
  \emph{Games and Economic Behavior}, vol.~29, no. 1-2, pp. 7--35, 1999.

\bibitem{cesa2006prediction}
N.~Cesa-Bianchi and G.~Lugosi, \emph{Prediction, learning, and games}.\hskip
  1em plus 0.5em minus 0.4em\relax Cambridge university press, 2006.

\bibitem{besbes2014stochastic}
O.~Besbes, Y.~Gur, and A.~Zeevi, ``Stochastic multi-armed-bandit problem with
  non-stationary rewards,'' \emph{Advances in neural information processing
  systems}, vol.~27, pp. 199--207, 2014.

\bibitem{jaksch2010near}
T.~Jaksch, R.~Ortner, and P.~Auer, ``Near-optimal regret bounds for
  reinforcement learning.'' \emph{Journal of Machine Learning Research},
  vol.~11, no.~4, 2010.

\bibitem{azar2017minimax}
M.~G. Azar, I.~Osband, and R.~Munos, ``Minimax regret bounds for reinforcement
  learning,'' in \emph{International Conference on Machine Learning}.\hskip 1em
  plus 0.5em minus 0.4em\relax PMLR, 2017, pp. 263--272.

\bibitem{fruit2018near}
R.~Fruit, M.~Pirotta, and A.~Lazaric, ``Near optimal exploration-exploitation
  in non-communicating markov decision processes,'' \emph{arXiv preprint
  arXiv:1807.02373}, 2018.

\bibitem{cappe2013supplement}
O.~Capp{\'e}, A.~Garivier, R.~Munos, and G.~Stoltz, ``Supplement to
  “{Kullback--Leibler} upper confidence bounds for optimal sequential
  allocation,'' 2013.

\bibitem{lin1991divergence}
J.~Lin, ``Divergence measures based on the {Shannon} entropy,'' \emph{IEEE
  Transactions on Information theory}, vol.~37, no.~1, pp. 145--151, 1991.

\bibitem{agrawal2020finite}
R.~Agrawal, ``Finite-sample concentration of the multinomial in relative
  entropy,'' \emph{IEEE Transactions on Information Theory}, vol.~66, no.~10,
  pp. 6297--6302, 2020.

\bibitem{dragomir2000some}
S.~S. Dragomir, M.~Scholz, and J.~Sunde, ``Some upper bounds for relative
  entropy and applications,'' \emph{Computers \& Mathematics with
  Applications}, vol.~39, no. 9-10, pp. 91--100, 2000.

\end{thebibliography}

\end{document}